\documentclass[1p]{elsarticle}

\usepackage[utf8]{inputenc} 
\usepackage[T1]{fontenc}    
\usepackage{hyperref}       
\usepackage{url}            
\usepackage{booktabs}       
\usepackage{amsfonts}       
\usepackage{nicefrac}       
\usepackage{microtype}      
\usepackage{xcolor}         
\usepackage{amsmath}
\usepackage{amsthm}
\usepackage[english]{babel}
\usepackage{caption}
\usepackage{subcaption}
\usepackage{multirow, float}
\usepackage{makecell}
\usepackage{graphicx}

\title{Real-time system optimal traffic routing under uncertainties -- Can physics models boost reinforcement learning?}

\author[1]{Zemian Ke}
\author[1]{Qiling Zou}
\author[1]{Jiachao Liu}
\author[1,2]{Sean Qian\corref{cor1}}
\cortext[cor1]{Corresponding author, seanqian@cmu.edu}

\affiliation[1]{organization={Department of Civil and Environmental Engineering, Carnegie Mellon University},
city={Pittsburgh},
country={USA}}
\affiliation[2]{organization={Heinz College, Carnegie Mellon University},
city={Pittsburgh},
country={USA}}

\newtheorem{theorem}{Theorem}
\newtheorem{remark}{Remark}
\newtheorem{lemma}{Lemma}

\begin{document}

\begin{abstract}
System optimal traffic routing can mitigate congestion by assigning routes for a portion of vehicles so that the total travel time of all vehicles in the transportation system can be reduced. However, achieving real-time optimal routing poses challenges due to uncertain demands and unknown system dynamics, particularly in expansive transportation networks. While physics model-based methods are sensitive to uncertainties and model mismatches, model-free reinforcement learning struggles with learning inefficiencies and interpretability issues. Our paper presents TransRL, a novel algorithm that integrates reinforcement learning with physics models for enhanced performance, reliability, and interpretability. TransRL begins by establishing a deterministic policy grounded in physics models, from which it learns from and is guided by a differentiable and stochastic teacher policy. During training, TransRL aims to maximize cumulative rewards while minimizing the Kullback–Leibler (KL) divergence between the current policy and the teacher policy. This approach enables TransRL to simultaneously leverage interactions with the environment and insights from physics models.  We conduct experiments on three transportation networks with up to hundreds of links. The results demonstrate TransRL's superiority over traffic model-based methods for being adaptive and learning from the actual network data. By leveraging the information from physics models, TransRL consistently outperforms state-of-the-art reinforcement learning algorithms such as proximal policy optimization (PPO) and soft actor-critic (SAC). Moreover, TransRL's actions exhibit higher reliability and interpretability compared to baseline reinforcement learning approaches like PPO and SAC.
\end{abstract}

\maketitle

\section{Introduction}

Traffic congestion in urban areas is one of the plagues of citizen's everyday life and it can cost huge economic loss \cite{jayasooriya2017measuring}. According to a report from Federal Highway Administration \footnote{https://ops.fhwa.dot.gov/aboutus/opstory.htm}(FHWA), ``roughly half of the congestion experienced by Americans happens virtually every day''.  This type of congestion is generated by unbalanced temporal and spatial distribution of traffic activities. Recent studies \cite{Chen2020, ke2023leveraging} have shown that congestion can be significantly mitigated by optimally guiding route choices of a small portion of travelers. The system optimal routing aims to minimize the total system cost by assigning routes for some travelers. There are many ways to practically implement and guide traffic routing. For instance, congestion toll \cite{yang1998departure, chen2018dyetc} changes users' routing choice by changing their perceived cost. Variable traffic signs (VMS)  provide real-time traffic information using road-side signals like LED signs \cite{emmerink1996variable, erke2007effects,arnott1991does,jacob2006automated} to affects users' routing choice. Route guidance and driver information systems (RGDIS) \cite{adler2002cooperative} directly provide route recommendations to users using in-vehicle devices such as infoentertainment system or cell phone applications. However, solving system-level optimal routing in real-time for large networks remains a big challenge. 

The optimal routing in real-time for large networks is challenging for three reasons. First, demands are uncertain. Travel demands stem from human activities that may have randomness, so it is almost unlikely to predict the travel demands accurately. Demands essentially affect traffic conditions, and the difference between the actual demands and the estimated demands can cause nontrivial estimation errors in traffic conditions. For example, \cite{qian2011computing} showed that congestion duration can change significantly with an incremental change in demands. Therefore, it is important to consider the demand uncertainties. From a theoretical perspective, \cite{waller2001evaluation} showed that using expected demands tends to overestimate the network performance in traffic assignment evaluation. \cite{do2012dynamic} also stated it is important to consider demand uncertainty for the congestion pricing problems. In this study, we explicitly include the demand uncertainty and regard demands as random variables that cannot be predicted precisely in advance.

Second, system dynamics may not be modeled perfectly. Though extensive traffic flow models (e.g., the celebrated LWR model \cite{lighthill1955kinematic, richards1956shock}, cell transmission models \cite{daganzo1994cell, daganzo1995cell}, and link queue model \cite{jin2021link}) have been proposed to approximate the system dynamics of transportation networks, no known flow model can precisely replicate the flow in networks consisting of both roads and intersections. Each model is associated with various assumptions and possible parameter estimation errors. For example, conventional kinematic models (\cite{lighthill1955kinematic, richards1956shock, daganzo1994cell, daganzo1995cell, jin2021link}) all assume a deterministic fundamental diagram to depict the relation between density and flow for mathematical tractability, while \cite{qu2017stochastic} states this relation can be better modeled by stochastic models. In addition,  estimation of the deterministic fundamental diagram parameters is not always accurate due to the randomness of traffic flow and transient flow state transitions. In this study, we acknowledge that the traffic flow model cannot be known precisely. Rather, there is a model mismatch between the assumed/adopted traffic models and the actual true system dynamics, and this mismatch is unkonwn in advance.

Third, real-time system optimal traffic routing is challenging especially in large networks due to complex interplay among traffic flow of various origins and destinations. \cite{pi2017stochastic} proposed a stochastic optimal real-time routing method for a two-route network. By abstracting a general commuting network into a two-route network, the analytical solution can be derived and approximated using dynamic programming. However, deriving the analytical solution of optimal traffic routing (i.e. path flow) for a large network is almost infeasible because the complexity of the problem and the dimension of decision variables increase exponentially as the network size. To that end, machine learning methods, like reinforcement learning, could be inefficient and even infeasible. The state space and action space increase exponentially as the network size increases, so finding an optimal policy from the exploded search space without any prior knowledge is inefficient and sometimes even infeasible. To our best knowledge, though \cite{Lazar2021} leveraged reinforcement learning to solve real-time optimal routing on networks with up to 41 links. Real-time optimal routing has not been solved by reinforcement learning on networks with more than hundreds of links. This study experiments on three networks with 2 links, 18 links, and 621 links, respectively.

Current methods for real-time optimal routing or general traffic control can be categorized into heuristic methods, model-based methods, and reinforcement learning-based methods. Heuristic methods use rule-of-thumb or feedback control to correct control errors to maintain desired system states. \cite{kachroo1998solution} formulated the real-time routing problem as a feedback control problem, and a feedback linearization technique is used to achieve a user equilibrium state. \cite{paz2009behavior} adopted a fuzzy control approach to determine real-time routing strategies to improve the network performance. Heuristic methods are interpretable and easy to implement, but they are reactive and only take action when control errors arise. Therefore, their performance is almost necessarily suboptimal.

Model-based methods typically first develop a traffic model or simulation to simulate the system dynamics, and then solve the optimal solutions through either heuristic algorithms (e.g., moving successive averages (MSA)) or deriving analytical closed-form solutions. For example, \cite{toledo2015simulation} used autoregressive integrated moving average (ARIMA) to predict demands, a binary logit model to predict traveler choices, and CTM to predict traffic flows. Based on these predictions, an optimization algorithm was adopted to set tolls to affect routing choices to optimize the objective function. \cite{tan2018hybrid} proposed a model predictive control method for dynamic pricing to reduce the total traveler delay. Similar to \cite{toledo2015simulation}, \cite{tan2018hybrid} used a logit choice model to predict traveler choices and CTM to predict traffic flows. Model-based methods are interpretable and can be proven to be optimal in theory. However, the inevitable model mismatch between the models and the actual system dynamics intrigues their performances in practice.

Reinforcement learning (RL) gained popularity in the traffic control domain because it's model-free and can learn optimal policies directly from data or the environment. \cite{pandey2020deep} considered a practical control scenario with multiple origins and destinations, partially observed network states, and stochastic demands. The problem was formulated as a partially observable Markov decision process (POMDP) and solved using RL. The experiments in \cite{pandey2020deep} showed RL outperformed feedback control. \cite{Lazar2021} considered a scenario where route choices of autonomous vehicles can be fully controlled to improve network efficiency, and the policy was learned using reinforcement learning. The experiments on simplified networks indicated the learned RL policy realized performances on par with the system optimum. However, to our best knowledge, RL has not been tested for optimal routing on large networks with more than hundreds of links. With the well-known curse of dimensionality, whether RL can solve optimal policies in large networks remains unknown.

In this study, the \textbf{reliability and interpretability} of the RL policies are the focus for its realism in real-world traffic operation practice. Most RL policies are not deterministic and the traffic environment is stochastic in nature, and it is clear that the performances of RL policies in practice are subject to variability. When implementing RL policies, one question out of interest is how poor the performance of RL policies can possibly be in the worst cases. In reality, traffic operation may not tolerate a poor performance for just one day, unlike RL applications in other domains, e.g. robotics, that may accept some time of online improvement through learning. In this work, we propose Conditional Travel Time Reduction at Risk (CTTRaR) to measure the reliability of a control method. CTTRaR compares the total travel time of the worst cases with a control method and the total travel time of the no-control scenarios. RL policies are regarded as reliable if the worst performance is still better than the no-control scenarios. In addition, the interpretability of the RL policies is a concern for practitioners. The objective of RL is to maximize the cumulative rewards with no specific constraints for the output actions. As a result, though RL leads to high cumulative rewards, the actions of RL might fluctuate over time, which makes RL policies less interpretable. Our general idea is to constrain the RL policies to not go too far from an interpretable ``teacher'' policy. For example, system optimum routing under traffic flow dynamics models can serve as a teacher policy. Those flow dynamics models may not be precisely calibrated (e.g., in a simulation environment) or the demand is stochastic, but its flow propagation mechanisms and the analytical solution can help to guide and learn interpretable RL policies.

Moreover, it is unclear when RL is superior to model-based methods and vice versa. A model-based method is extensively compared with model-free RL methods to shed light on the performance between model-based methods and learning-based methods. The results indicate RLs outperform the model-based method when the demand uncertainty is large or/and the model mismatch is significant, and vice versa. When the network size is large, RLs sometimes struggle to find a reasonable policy that works precisely for the unknown or unexplored system states and dynamics, but model-based can perform well especially when the model mismatch is insignificant. With the trade-off between the model-based method and model-free RL, one interesting question is can we combine the advantages of both types of methods?

We propose a general reinforcement learning framework that can couple different types of well-established transportation methods (e.g., heuristics, model-based, or machine learning-based) with RL. In this study, we focus on an example of coupling the model-based method with RL. Our proposed reinforcement learning framework differentiates from most previous reinforcement learning algorithms including DDPG \cite{schulman2017proximal}, PPO \cite{schulman2017proximal}, and SAC \cite{Haarnoja2018} by leveraging information from transportation methods, named \textbf{Trans}portation-informed \textbf{R}einforcement \textbf{L}earning (TransRL). TransRL is able to learn from the environment and the traffic model simultaneously. On one hand, even with a model mismatch, the information from the models is not unuseful and can help RL narrow down the search space. On the other hand, with the ability to learn from the environment directly, RL is able to learn a better policy than the model-based method by implicitly correcting the model mismatch between the models and the actual system dynamics.

The contributions of this paper are summarized as follows.
\begin{itemize}
    \item We solve real-time system optimal routing in sizable transportation networks. The model is based on a realistic setting where only a few links are observed and a portion of vehicles can be influenced with their respective routing guidance.
    \item We relax the assumptions of known travel demands and accurate traffic models in other RL models. Instead, travel demands are assumed to follow a time-dependent Gaussian distribution with means of historical average demands. Moreover, there are model mismatches between the offline trained accessible traffic models and the online test (unknown and true) system dynamics, which stem from lack of knowledge, model estimation errors, or unexpected incidents in real world.
    \item Model-free RLs are compared with a traffic model-based method under various levels of model mismatches. This comparison provides insights into when a traffic model-based method is superior and when RLs are preferable.
    \item Ultimately, we proposed a novel RL framework TransRL that combines RL with traffic models. With the ability to learn from the environment and traffic models simultaneously, TransRL learns more efficiently than model-free RLs and is more adaptive than the traffic model-based method. More importantly, the actions of TransRL are more reliable and interpretable than model-free RLs.
    \item Reliability and interpretability of TransRL are the focus. In order to use TransRL in practice, its performance at any stage throughout the process of online learning cannot go below the network performance without any traffic control/management measures. TransRL's policies are also designed to ensure they approximately follow the guidance of model-based optimal flow solutions, which ensure its interpretability. 
\end{itemize}

The rest of the paper is organized as follows. Literature related to vehicle routing and RL-based control methods is discussed in section \ref{sec: related work}. In section \ref{sec: TransRL}, we introduce our framework TransRL and show it can be proved to converge under conditions of finite discrete state and action space. We subsequently elaborate on how TransRL can be trained in continuous state space and action space by using neural networks to approximate value functions and policies. The problem formulation of the real-time optimal routing problems is included in section \ref{sec: problem formulation}. Section \ref{sec: problem formulation} then presents how to solve the optimal routing problems using a model-based method, model-free RL, and the proposed TransRL. The considered methods are compared with various experiment settings on three networks in section \ref{sec: experiments}. Section \ref{sec: conclusion} concludes our findings and suggests potential directions for future work.

\section{Related works}
\label{sec: related work}
Quite a few literatures investigated providing real-time route information to travelers through variable message signs (e.g., \citet{messmer1998automatic, wunderlich2000link}), and in-vehicle routing mechanisms \citep{aerde1988individual}. However, assigning routes or providing route information to improve the network performance is less studied. \cite{adler2005multi} used cooperative distributed multi-agent systems to explore the interactions between route information providers and travelers, and it was found that negotiation between information providers and travelers can improve the network performance. \cite{paz2009behavior} studied affecting traveler routing behaviors via providing real-time routing information, and a fuzzy control approach was proposed to determine the information strategy in order to enhance the network performance. \cite{du2015coordinated} proposed a coordinated online in-vehicle routing scheme with intentional information provision perturbation (CRM-IP), which leverages the bounded rationality of travelers to shape traveler routing behaviors so that the system optimality and user optimality are balanced. \cite{Lazar2021} studied a scenario where route choices of autonomous vehicles can be fully controlled in a centralized manner to improve network performance. The policy learned by RL performed close to a theoretically optimal solution in networks composed of up to dozens of links.

In the general traffic control domain, various analytical methods have been proposed, such as dynamic programming for optimal routing on a 2-route network \cite{pi2017stochastic}, max-pressure for signal control \cite{2022integrating, 2024ped, 2023boosting}, and feedback control for ramp metering \cite{wang2001feedback, zhao2020fuzzy}, to name a few. The analytical methods develop control policies through mathematical derivations on the top of model assumptions. These derived control policies often show desirable properties, such as convergence guarantee (i.e., reliability) and interoperability. However, real environments may change over time. Because of a lack of learning ability, these control policies cannot automatically adapt to changing environments, which deteriorates control performances.

Reinforcement learning algorithms have gained popularity in solving real-time control problems because of the characteristics of being model-free and learnable while capable of taking proactive optimal actions even in uncertain environments \cite{Haydari2022, ParvezFarazi2021}. RL algorithms have been extensively applied to various traffic control problems, including vehicle routing \cite{kullman2022dynamic, Singh2019}, autonomous driving \cite{ye2019automated, zhu2018human}, traffic signal control \cite{prashanth2010reinforcement, li2021network}, ramp metering \cite{yang2019deep}, variable speed limit control \cite{ke2020enhancing}, and congestion pricing \cite{pandey2020deep}. RL used in transportation can be categorized into model-free RL (e.g., \cite{li2021network, Haydari2022}) and RL with prior knowledge (e.g., \cite{ke2020enhancing, Han2022, Han2022a, Chen2022, Bai2022}). Model-free RL learns from scratch by interacting with the environment, so the RL agent explores the search space by mostly taking random actions during the early stage of the training. As a consequence, the performance of model-free RL is not guaranteed and is even worse than the non-control case before the convergence. This is particularly problematic in the traffic operation domain, because any traffic management measure cannot afford to under-perform under general public's expectation, even for just a few days. In other words, practically it would be impossible to take a number of days before RL starts to show system benefits. Our paper can help address this issue by guiding RL online learning with policies derived from physics models. 

Compared with model-free RL, RL with prior knowledge is more data-efficient by utilizing prior knowledge to narrow down the search space or start with a policy better than random initialization. \cite{ke2020enhancing} utilized transfer learning to transfer a policy trained on a source scenario to multiple target scenarios. Though the fundamental diagrams in the target scenarios differed from those in the source scenario, the transfer learning significantly shortened the training process in the target scenarios. \cite{Han2022, Han2022a} augmented field data using traffic flow models. From the augmented data, RL was trained offline, and then the learned policies were implemented to acquire new field data. This process repeated such that RL kept learning from the environment. \cite{Bai2022} designed a hybrid reinforcement learning framework that combines a rule-based strategy and reinforcement learning to optimize the eco-driving strategy. Most of the time, the vehicle is controlled by RL policies. When the stop warning is activated, the rule-based strategy will take control and ensure the vehicle stops safely. \cite{Chen2022} proposed an RL-based framework for end-to-end autonomous driving. While learning to behave optimally, the proposed framework also learns a latent environment model that predicts the state of the environment in a low-dimensional latent space. The latent model greatly reduces the sample complexity by learning the latent states of the high-dimensional observations. Though RL has been extensively studied in engineering problems (e.g., \cite{2020QL, 2020MARL, 2023RL}), the actions of RL policies are not constrained to improve the reliability and interpretability of the RL policies.


\section{TransRL}
\label{sec: TransRL}
\subsection{Preliminaries for RL}
The problem studied in RL can be formulated as a Markov decision process (MDP). The state space and the action space are denoted as $\mathcal{S}$ and $\mathcal{A}$. After applying an action $a_t \in \mathcal{A}$ at current state $s_t \in \mathcal{S}$, the next state $s_{t+1} \in \mathcal{S}$ of the environment is determined by a state transition probability function $p(s_{t+1} \mid s_t, a_t)$, and a reward $r_t$ is produced by the environment. The return is the discounted sum of rewards in the whole horizon $\sum_t \gamma^t r_t$, where $\gamma$ is a discount factor making a trade-off between short-term rewards and long-term rewards. A policy $\pi (a_t \mid s_t)$ determines the probability distribution of actions $a_t$ given a state $s_t$. We denote state-action marginals induced by a policy $\pi (a_t \mid s_t)$ as $\rho_\pi (s_t, a_t)$. 

The objective of reinforcement learning is to find an optimal policy $\pi^*$ such that the expected return is maximized.
\begin{align}
\label{equation: RL objective}
    \pi^* = \arg \max_\pi \mathbb{E}_{(s_t, a_t)\sim \rho_\pi} \left[\sum_t \gamma^t r_t\right]
\end{align}

This study focuses on actor-critic methods \cite{fujimoto2018addressing, lillicrap2015continuous}, which are a combination of Q-learning \cite{watkins1989learning} and policy optimization. Generally, actor-critic methods use a critic network parameterized as $\theta$ to estimate the expected return of state-action pairs with a policy $\pi$ (i.e., Q-values $Q_\theta^\pi(s_t, a_t)$), and an actor network parameterized as $\phi$ to output actions given states (i.e., $\pi_\phi (a_t \mid s_t)$).

Both the actor network and the critic network are trained using data from an experience replay buffer containing transitions $(s_t, a_t, s_{t+1}, r_t)$. The actor network $\pi_\phi (a_t \mid s_t)$ is updated by gradient descent with a loss function based on Q-values.
\begin{align}
    J_\pi (\phi) =  \mathbb{E}_{a_t \sim \pi_\phi} \left[ - Q_\theta(s_t, a_t)  \right] 
\end{align}
The critic network is updated by the temporal difference error based on the Bellman equation, and the loss function for the critic network is given by

\begin{align}
    J_Q (\theta)  =  \left[ Q_\theta (s_t, a_t) - (r_t + \gamma \left( Q_{\bar{\theta}}(s_{t+1}, a_{t+1}) \right) ) \right]^2, a_{t+1} \sim \pi_{\bar{\phi}} (s_{t+1})
\end{align}

where $Q_{\bar{\theta}}$ and $\pi_{\bar{\phi}}$ are a target critic network and a target actor network respectively, which are used to stabilize the training. The actor and critic networks are updated at every learning step, while the target actor and critic networks are updated periodically by copying the weights of the actor and critic networks respectively.

\subsection{TransRL models}
Our proposed reinforcement learning framework differentiates from most previous reinforcement learning algorithms (e.g., DDPG \cite{schulman2017proximal}, PPO \cite{schulman2017proximal}, and SAC \cite{Haarnoja2018}) by leveraging information from transportation methods, named \textbf{Trans}portation-informed \textbf{R}einforcement \textbf{L}earning (TransRL). First, a neural network is used to learn a stochastic and differentiable teacher policy from a well-established transportation method. Then, during the training or testing of TransRL, aside from interactions with the environment, TransRL also learns from the teacher policy by comparing the current policy with the teacher policy. The divergence between the two policies, combined with rewards, is passed to TransRL for learning. The whole process is illustrated in Figure \ref{fig:concept}.

\begin{figure}[H]
    \centering
    \includegraphics[width=\textwidth]{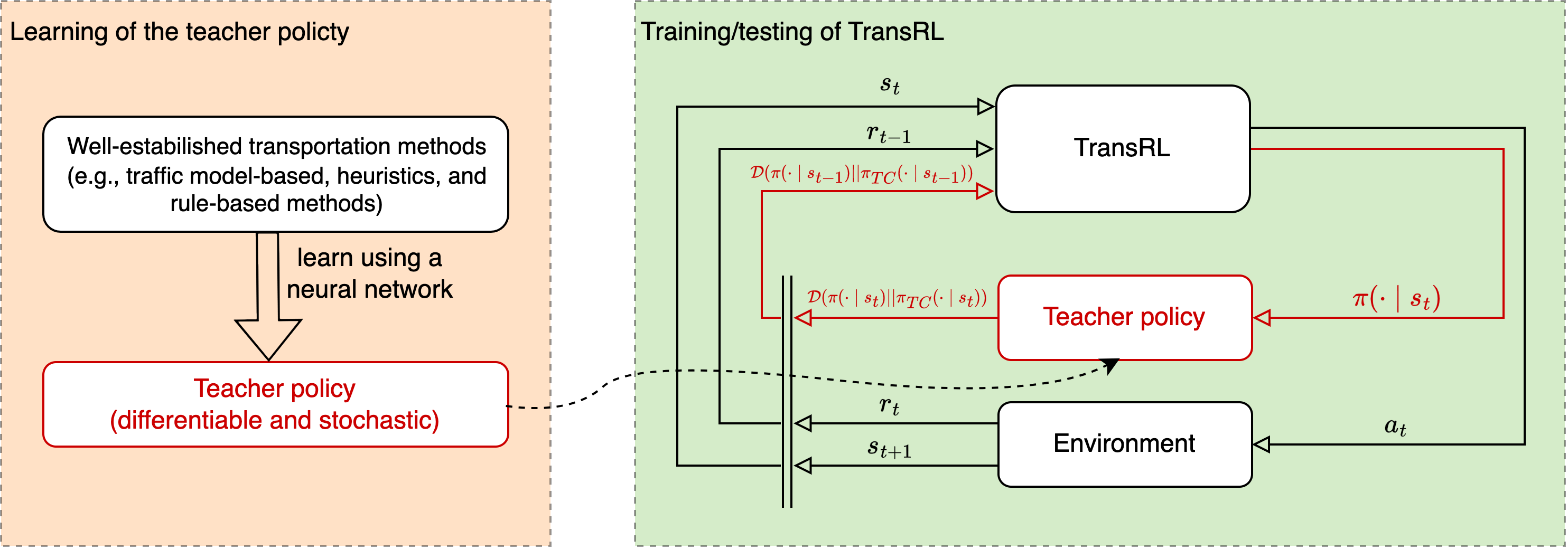}
    \caption{TransRL concept diagram}
    \label{fig:concept}
\end{figure}

\subsubsection{Teacher policy}
\label{sec: teacher policy}
We assume there exist well-established transportation models and methods that can output actions. The choice of transportation domain models can be very flexible, and it can be traffic model-based, heuristics, or rule-based methods. Compared with general learning-based methods, which work like a black box, transportation domain models are more reliable and interpretable as they incorporate domain knowledge and physics information that would guide the RL agent to learn effectively.

Then, we develop a stochastic and differentiable teacher policy based on a transportation domain model. The teacher policy is a probability distribution for actions given a state. The action distribution concentrates on the action output by the transportation domain model. The concentration level of the distribution can be tuned by an unreliability parameter $\sigma$. Figure \ref{fig: unreliability parameter concept} shows a simple case with an action dimension of 1 and the range of action is $[0, 1]$ where the action represents the portion of vehicles to be diverted from the most preferred route. As Figure \ref{fig: unreliability parameter concept} shows, a smaller unreliability parameter leads to a distribution more concentrated on the action from the deterministic transportation domain model. Essentially, the teacher policy is a prior distribution of actions for states. We will provide an example of a teacher policy in the later section \ref{sec: TransRL for routing}.

\begin{figure}[H]
    \centering
    \includegraphics[width=0.8\textwidth]{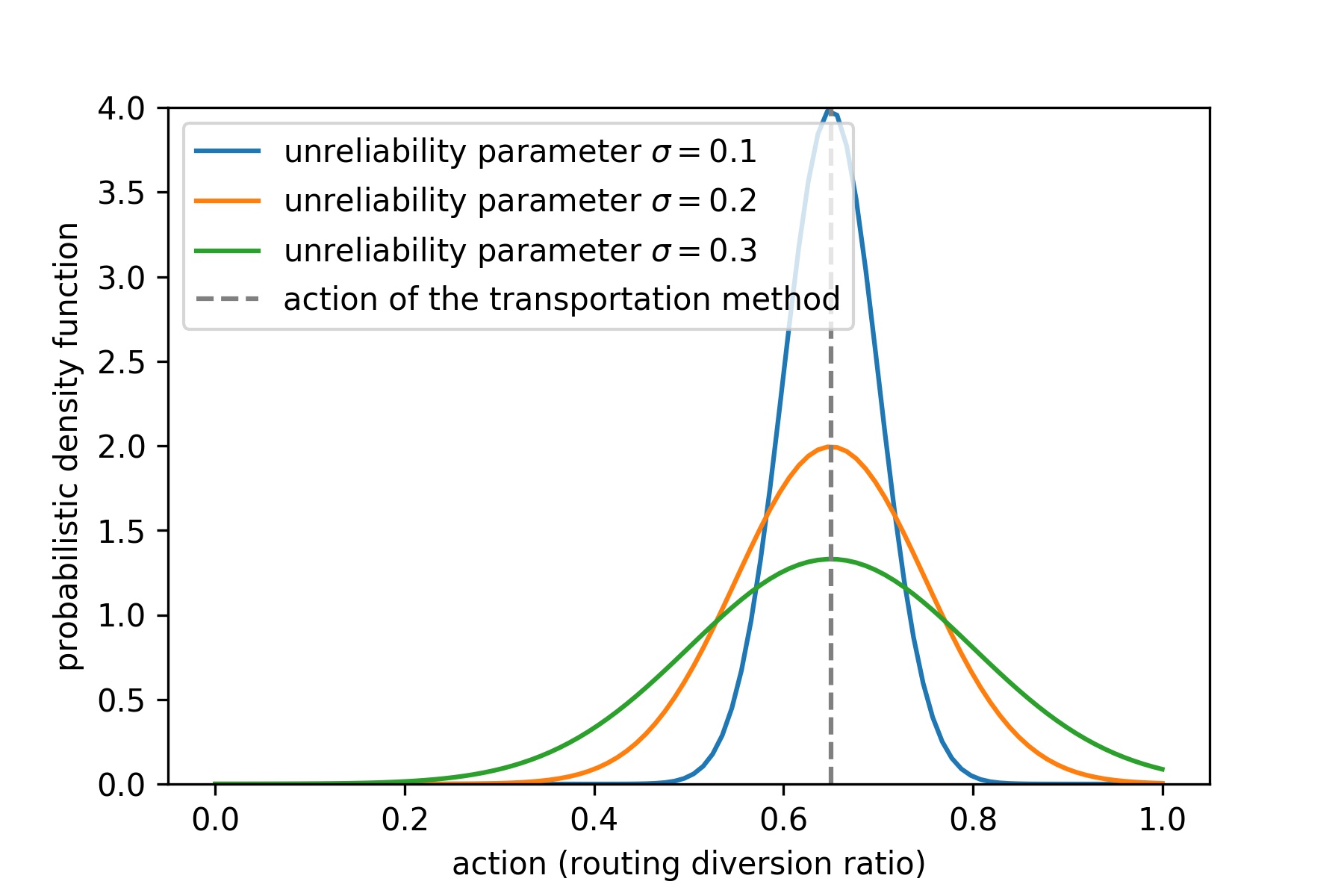}
    \caption{The teacher policy}
    \label{fig: unreliability parameter concept}
\end{figure}

The benefits of making the teacher policy stochastic and differentiable are three-fold. First, the unreliability parameter enables tuning how much the TransRL agent relies on the transportation domain model. Also, a stochastic teacher policy instructs the TransRL agent how explorative it should be. A teacher policy with concentrated action distribution results that the TransRL agent explores less and concentrates on actions near the action derived from transportation domain models, while a teacher policy with a more flat distribution encourages the TransRL agent to explore more on the whole action space. At last, a differentiable teacher policy enables the gradient propagation, which makes possible the TransRL learning from the teacher policy.

\subsubsection{The augmented objective function}
To enable TransRL's learning from the teacher policy, the objective of TransRL is to maximize the cumulative rewards while minimizing the differences between the divergence between the learned policy and the teacher policy. Therefore, the objective function of TransRL is different from the objective function of most previous reinforcement learning algorithms as equation \eqref{equation: RL objective} shows. This novel objective fundamentally changes the exploration behaviors and the learning process of the reinforcement learning agent. Specifically, the objective function of TransRL is given by
\begin{equation}
    J(\pi) = \mathbb{E}_{(s_t, a_t)\sim \rho_\pi} \left[ \sum_t \gamma^t \left( r_t - \alpha \mathcal{D} \left(\pi (\cdot \mid s_t) || \pi_{TC} (\cdot \mid s_t) \right) \right) \right]
\end{equation}
where $\alpha$ is the temperature parameter that makes a trade-off between the reward and the divergence term. The choice of divergence function is relatively flexible. Without loss of generality, we adopt the Kullback–Leibler (KL) divergence. Then, the augmented objective becomes
\begin{align}
 J(\pi) & =  \mathbb{E}_{(s_t, a_t)\sim \rho_\pi} \left[ \sum_t \gamma^t \left( 
r_t - \alpha D_{KL} \left( \pi (\cdot \mid s_t) || \pi_{TC} (\cdot \mid s_t) \right)
 \right) \right] \\
 \label{eq: TransRL obj}
  & =  \mathbb{E}_{(s_t, a_t)\sim \rho_\pi} \left[ \sum_t \gamma^t \left( 
r_t + \alpha \mathcal{H} \left( \pi (\cdot \mid s_t)  \right) + \alpha \mathbb{E}_{a \sim  \pi (\cdot \mid s_t) } \left[ \log \pi_{TC} (a \mid s_t) \right]
 \right) \right] 
\end{align}

\begin{theorem}
\label{theorem: SAC and TransRL}
The maximum entropy reinforcement learning SAC \cite{Haarnoja2018} is equivalent to TransRL with a teacher policy of uniform distribution.
\end{theorem}

\begin{proof}
     If $\pi_{TC}$ is a uniform distribution, $\pi_{TC} (a \mid s_t)$ is a constant regardless of $a$, so $\mathbb{E}_{a \sim  \pi (\cdot \mid s_t) } \left[ \log \pi_{TC} (a \mid s_t) \right]$ is also a constant regardless of $\pi$. As a result, we can remove $\mathbb{E}_{a \sim  \pi (\cdot \mid s_t) } \left[ \log \pi_{TC} (a \mid s_t) \right]$ from the objective function \eqref{eq: TransRL obj}, and then it becomes
\begin{align}
 J(\pi) =  \mathbb{E}_{(s_t, a_t)\sim \rho_\pi} \left[ \sum_t \gamma^t \left( 
r_t + \alpha \mathcal{H} \left( \pi (\cdot \mid s_t)  \right)
 \right) \right] 
\end{align}
which is the objective of SAC \cite{Haarnoja2018}.
\end{proof}

\begin{remark}
Including the KL divergence term in the objective results that TransRL becomes more reliable and interpretable than classical RL. With this divergence constraint, the policy $\pi$ does not go too far away from the teacher policy $\pi_{TC}$ that is assumed to be reliable and interpretable.
\end{remark}

\subsubsection{Policy iteration}
The optimal policy is solved using policy iteration. The policy iteration includes, (1) the policy evaluation step where the value functions are estimated given a policy, and (2) the policy improvement step where the policy is updated to increase the values given the value functions. By repeating the policy evaluation and the policy improvement, the policy is guaranteed to converge to one of the optimal policies that maximize values.

TransRL uses auxiliary Q-values according to the augmented objective function such that maximizing auxiliary Q-values is equivalent to maximizing the augmented objective. The auxiliary Q-values are given by
\begin{align}
\label{eq: auxiliary Q value definition}
    Q(s_t, a_t) = r_t + \gamma \mathbb{E}_{s_{t+1} \sim p} \left[ V(s_{t+1}) \right]
\end{align}
where
\begin{align}
    V(s_{t+1}) = \mathbb{E}_{a_{{t+1}} \sim \pi} \left[ Q(s_{t+1}, a_{t+1}) + \alpha \log \frac{\pi_{TC} (a_{t+1} \mid s_{t+1})}{\pi (a_{t+1} \mid s_{t+1})} \right]
\end{align}
is the auxiliary state value. 

For the policy evaluation, The auxiliary Q-values are updated iteratively by a modified Bellman equation
\begin{equation}
\label{eq: Bellman}
    Q^{k+1}(s_t, a_t) \leftarrow r_t + \gamma \mathbb{E}_{s_{t+1} \sim p} \left[ V^{k}(s_{t+1}) \right]
\end{equation}
With the above update rule, the auxiliary Q-values can be proved to converge to the unique auxiliary Q-values of the policy $\pi$ as follows.
\begin{lemma}[Policy evaluation convergence]
\label{lemma: policy evaluation}
    Starting with a initial auxiliary Q-values function $Q^0: \mathcal{S} \times \mathcal{A} \rightarrow \mathbb{R}$ with $|\mathcal{A}| < \infty $, update $Q^0$ iteratively using equation \eqref{eq: Bellman}. As $k \rightarrow \infty$, $Q^k$ converges to the unique auxiliary Q-values function of $\pi$.
\end{lemma}

\begin{proof}
    Let's define an augmented reward
    \begin{equation}
    \label{eq: augmented reward}
        r_t^\pi (s_t, a_t) = r_t(s_t, a_t) + \gamma \alpha \mathbb{E}_{a_{t+1} \sim \pi} \left[ \log \frac{\pi_{TC} (a_{t+1} \mid s_{t+1})}{\pi (a_{t+1} \mid s_{t+1})} \right]
    \end{equation}
    Then, equation \eqref{eq: Bellman} becomes
    \begin{equation}
    \label{eq: Bellman 2}
        Q^{k+1}(s_t, a_t) \leftarrow r_t^\pi (s_t, a_t) + \gamma \mathbb{E}_{s_{t+1} \sim p, a_{t+1} \sim \pi} \left[ Q^k (s_{t+1}, a_{t+1}) \right]
    \end{equation}
    As $|\mathcal{A}| < \infty$, the KL divergence in equation \eqref{eq: augmented reward} is bounded, so $r_t^\pi (s_t, a_t)$ is bounded. Then, one can apply the contraction mapping theorem to prove $Q^k$ converge to the unique auxiliary Q-values, which is the same as the proof for standard reinforcement learning \cite{sutton1998reinforcement}.
\end{proof}

For the policy improvement, the policy is updated such that the auxiliary Q-values of the new policy are higher than the old policy. Similar to \cite{Haarnoja2018}, we use a KL divergence between the old policy and the exponential of the auxiliary Q function, and the new policy is obtained by minimizing the KL divergence as follows.

\begin{align}
\label{eq: policy improvement}
    \pi_{new} = \arg \min_{\pi \in \Pi} D_{KL} \left( \pi(\cdot \mid s_t) \mid \mid \frac{\pi_{TC} (\cdot \mid s_t)\exp{(Q^{\pi_{old}} (s_t, \cdot)/\alpha)}}{\mathcal{Z}^{\pi_{old}}(s_t)} \right)
\end{align}
where $\mathcal{Z}^{\pi_{old}}(s_t)$ is the partition function ensuring the right part within the KL divergence is a probability distribution. Though $\mathcal{Z}^{\pi_{old}}(s_t)$ is intractable and infeasible for large state spaces, it is a constant and can be ignored when calculating gradients with respect to the policy. With this policy update rule, we can prove the new policy is better than the old policy with respect to auxiliary Q-values as follows.

\begin{lemma}[Policy improvement convergence]
\label{lemma: policy improvement}
    For $\pi_{old} \in \Pi$, $\pi_{new}$ is the optimal solution of the problem defined in equation \eqref{eq: policy improvement}. With $|\mathcal{A}| < \infty$, $Q^{\pi_{new}} (s_t, a_t) \geq Q^{\pi_{old}} (s_t, a_t), \forall (s_t, a_t) \in \mathcal{S} \times \mathcal{A}$.
\end{lemma}

\begin{proof}
    Denote the KL divergence term in equation \eqref{eq: policy improvement} as
    \begin{equation}
        J_{\pi_{old}} (\pi) = D_{KL} \left( \pi(\cdot \mid s_t) \mid \mid \frac{\pi_{TC} (\cdot \mid s_t)\exp{(Q^{\pi_{old}} (s_t, \cdot)/\alpha)}}{\mathcal{Z}^{\pi_{old}}(s_t)} \right)
    \end{equation}
    Then, as $\pi_{new}$ is a minimizer of the KL divergence, we have
    \begin{align}
        J_{\pi_{old}} (\pi_{old}) &\geq J_{\pi_{old}} (\pi_{new}) \\
        \begin{split}
            \Rightarrow \; \mathbb{E}_{a_t \sim \pi_{old}} & \left[\log \pi_{old} (a_t \mid s_t) - \log \pi_{TC} (a_t \mid s_t) - Q^{\pi_{old}}(a_t, s_t) / \alpha + \log \mathcal{Z}^{\pi_{old}} (s_t)  \right] \\
            & \geq \mathbb{E}_{a_t \sim \pi_{new}} \left[\log \pi_{new} (a_t \mid s_t) - \log \pi_{TC} (a_t \mid s_t) - Q^{\pi_{old}}(a_t, s_t) / \alpha + \log \mathcal{Z}^{\pi_{old}} (s_t)  \right]  \\
            \Rightarrow \; \mathbb{E}_{a_t \sim \pi_{old}} & \left[Q^{\pi_{old}}(a_t, s_t) + \alpha \log \frac{\pi_{TC} (a_t \mid s_t)}{\pi_{old} (a_t \mid s_t)}  \right] \\
            & \leq \mathbb{E}_{a_t \sim \pi_{new}} \left[Q^{\pi_{old}}(a_t, s_t) + \alpha \log \frac{\pi_{TC} (a_t \mid s_t)}{\pi_{new} (a_t \mid s_t)}  \right]
        \end{split}\\
        \label{eq: value inequality}
       \Rightarrow \; V^{\pi_{old}} (s_t) & \leq \mathbb{E}_{a_t \sim \pi_{new}} \left[Q^{\pi_{old}}(a_t, s_t) + \alpha \log \frac{\pi_{TC} (a_t \mid s_t)}{\pi_{new} (a_t \mid s_t)}  \right]
    \end{align}
    Now, follow equation \eqref{eq: auxiliary Q value definition} to replace $Q^{\pi_{old}}$ with a formulation of $V^{\pi_{old}}$, and then apply the inequality \eqref{eq: value inequality}. By conducting these two steps alternatively, we obtain
    \begin{align}
        Q^{\pi_{old}} (s_t, a_t) & = r(s_t, a_t) + \gamma \mathbb{E}_{s_{t+1} \sim p} \left[V^{\pi_{old}}(s_{t+1}) \right] \\
        & \leq r(s_t, a_t) + \gamma \mathbb{E}_{s_{t+1} \sim p} \left[\mathbb{E}_{a_t+1 \sim \pi_{new}} \left[Q^{\pi_{old}}(a_{t+1}, s_{t+1}) + \alpha \log \frac{\pi_{TC} (a_{t+1} \mid s_{t+1})}{\pi_{new} (a_{t+1} \mid s_{t+1})}  \right] \right] \\
        & \cdots  \\
        & \leq r(s_t, a_t) + \gamma \mathbb{E}_{s_{t+1} \sim p} \left[V^{\pi_{new}}(s_{t+1}) \right] \\
        & = Q^{\pi_{new}} (s_t, a_t)
    \end{align}
    
\end{proof}

The policy iteration alternates between the policy evaluation and the policy improvement. Finally, we can prove the policy iteration will converge to an optimal policy that maximizes the auxiliary Q-values as follows.
\begin{theorem} [Policy iteration convergence]
    With $|\mathcal{A}| < \infty$, starting from any policy $\pi^0 \in \Pi$, conduct the policy evaluation and the policy improvement iteratively. The policy will converge to an optimal policy $\pi^*$ such that $Q^{\pi^*} (s_t, a_t) \geq Q^\pi (s_t, a_t), \forall (s_t, a_t) \in \mathcal{S} \times \mathcal{A} , \pi \in \Pi$.
\end{theorem}

\begin{proof}
    According to Lemma \ref{lemma: policy improvement}, at each policy iteration step, we have $Q^{\pi_{new}} \geq Q^{\pi_{old}}$. As the augmented rewards are bounded, $Q^{\pi}$ is also bounded. As a result, $Q^{\pi}$ will converge to a certain point, where both the auxiliary Q-values and the policy converge, denoted as $Q^{\pi^*}$. Then, we get $J_{\pi^*} (\pi^*) \leq J_{\pi^*} (\pi), \forall \pi \in \Pi$. Similar to the proof of Lemma \ref{lemma: policy improvement}, we have $Q^{\pi^*} (s_t, a_t) \geq Q^\pi (s_t, a_t), \forall (s_t, a_t) \in \mathcal{S} \times \mathcal{A} , \pi \in \Pi$.
\end{proof}

Note the above policy iteration process only works for tabular cases where the state space and the action space are discrete. In more general cases where state variables and actions can be continuous. We can use approximations (e.g., neural networks) to approximate auxiliary Q-values. In addition, at the policy evaluation or the policy improvement step, running until convergence is computationally expensive. In the next section, we will introduce how TransRL is trained in practice, which is suitable for a continuous state and action space.

\subsubsection{Training of TransRL}
\label{traning of TransRL}
To accommodate large continuous state space and action space, we use neural networks to approximate the auxiliary Q-values/Q-function and the policy. The auxiliary Q-function is parameterized by the neural network parameters $\theta$, which is denoted as $Q_\theta (s_t, a_t)$. Similarly, the parameterized policy is $\pi_\phi (a_t \mid s_t)$, where $\phi$ are the parameters of the policy neural network.

For the Q-function network, the parameters are updated to minimize the Bellman residual. More specifically, the loss of the Q-function network is given by:

\begin{align}
    J_Q (\theta)  = \mathbb{E}_{s_t, a_t \sim \mathcal{D}} \left[ \left( Q_\theta (s_t, a_t) - \hat{Q}(s_t, a_t) \right)^2 \right]
\end{align}

where
\begin{align}
\label{Q value target}
    \hat{Q}(s_t, a_t) = r_t + \gamma \left( Q_{\bar{\theta}}(s_{t+1}, a_{t+1}) + \alpha \log \frac{\pi_{TC} (a_{t+1} \mid s_{t+1})}{\pi_\phi (a_{t+1} \mid s_{t+1})} \right), \, a_{t+1} \sim \pi_\phi (\cdot \mid s_{t+1})
\end{align}

where $Q_{\bar{\theta}}$ is the target Q-function whose parameters are exponentially moving average of parameters of $Q_{\theta}$. The use of target networks is able to stabilize the training process \cite{mnih2015human}. Then, the Q-function network is trained via stochastic gradient descent methods with gradient

\begin{align}
\label{Q function gradient}
    \nabla_\theta J_Q(\theta) = 2 \nabla_\theta Q_\theta (s_t, a_t) \left( Q_\theta (s_t, a_t) - \hat{Q}(s_t, a_t) \right)
\end{align}

The policy network is trained by minimizing the KL divergence in equation \eqref{eq: policy improvement}.

\begin{align}
    J_\pi (\phi) = \mathbb{E}_{s_t \sim \mathcal{D}} \left[  \mathbb{E}_{a_t \sim \pi_\phi} \left[ \log \pi_\phi (a_t \mid s_t) - \log \pi_{TC} (a_t \mid s_t) - Q_\theta(s_t, a_t)/\alpha  \right] \right]
\end{align}

To enable back-propagation in updating $\phi$, the reparameterization trick is used to generate action. 


\begin{align}
    a_t =f_\phi (\epsilon_t; s_t)
\end{align}
where $\epsilon_t$ is independent random noise sampled from a Gaussian distribution. Afterwards, the loss function of the policy network becomes

\begin{align}
    J_\pi (\phi) = \mathbb{E}_{s_t \sim \mathcal{D}, \epsilon_t \sim \mathcal{N} }  \left[ \log \pi_\phi (f_\phi (\epsilon_t; s_t) \mid s_t) - \log \pi_{TC} (f_\phi (\epsilon_t; s_t) \mid s_t) - Q_\theta(s_t, f_\phi (\epsilon_t; s_t))/\alpha  \right] 
\end{align}

Finally, the parameters $\phi$ can be updated via stochastic gradient methods with the gradient of

\begin{align}
\label{policy gradient}
    \nabla_\phi J_\pi (\phi) =\nabla_\phi \log \pi_\phi (a_t \mid s_t) + \left( \nabla_{a_t} \log \pi_\phi (a_t \mid s_t) -  \nabla_{a_t} \log \pi_{TC} (a_t \mid s_t) - \nabla_{a_t} Q_\theta (s_t, a_t) \right) \nabla_\phi f_\phi (\epsilon_t; s_t)
\end{align}

To mitigate the overestimation bias in the Q-values estimation, we use the clipped double Q-learning trick in our practical algorithm \cite{fujimoto2018addressing}. Specifically, we use two Q-functions $Q_{\theta_i}$ that are trained independently, and two target Q-functions $Q_{\bar{\theta}_{i}}$. Therefore, the target Q-function in equation \eqref{Q value target} is the minimum of the two target Q-functions, and the Q-function in equation \eqref{policy gradient} is the minimum of the two Q-functions.






\section{The real-time system optimal traffic routing problems}
\label{sec: problem formulation}
Suppose a morning commuting scenario on a general road network with vertices, $v \in \mathcal{V}$, and links, $l \in \mathcal{L}$. The origins (O) of travel demands are $g \in \mathcal{G}$, and the destinations (D) are $ e \in \mathcal{E}$. The set of paths from an origin $g$ to a destination $e$ is denoted as $\mathcal{P}^{ge}$. For small networks, the paths between $g$ and $e$ can be enumerated, so the path set $\mathcal{P}^{ge}$ contains all non-cyclic paths $p_i^{ge}$ between $g$ and $e$. For large networks, the paths cannot be enumerated, so $\mathcal{P}^{ge}$ contains a small portion of all possible paths. For example, each path set $\mathcal{P}^{ge}$ includes k shortest or most frequently used paths of each OD pair, or paths varying only by major alternative routes. The demands are assumed to be time-dependent and stochastic. The whole time horizon is divided into multiple time intervals, and each time interval is denoted by $t$. The demand of OD pair $ge$ during time interval $t$ is $q_t^{ge}$, which follows a Gaussian distribution $q_t^{ge} \sim \mathcal{N} (\mu(q_t^{ge}), \sigma(q_t^{ge}))$. 

The environment $\mathcal{M}$ includes two components: (1) the road network, including the topology of the network, and the link properties (i.e., lengths, the numbers of lanes, free flow speeds, and capacities); and (2) the traffic flow dynamics, which is Cell Transmission Model (CTM) \cite{daganzo1994cell, daganzo1995cell} and node models \cite{nie2010solving} in our study. CTM and node model essentially determine how the environment state evolves given the current state, input demands, and control actions.

The objective of real-time system optimal traffic routing problems \cite{pi2017stochastic} is to assign all vehicles to paths so that the total travel time of all travelers is minimized. Since the travel demands are stochastic, we cannot expect the total number of vehicles to be assigned, nor the number of vehicles for each path in the path set. Therefore, the control actions in our formulation are the assignment ratios across paths. The assignment ratio of path $p_i^{ge}$ between OD pair $ge$ during time interval $t$ is denoted as $\kappa_{i, t}^{ge}$ and $\sum_i \kappa_{i,t}^{ge} = 1$. At the beginning of each time interval $t$, the control algorithm decides the path ratios $\kappa_{i, t}^{ge}$, so there are $q_t^{ge} \kappa_{i, t}^{ge}$ vehicles to be routed to the path $p_i^{ge}$ during time interval $t$. The control algorithm aims at minimizing the total travel time of all travelers within the whole time horizon by determining path ratios.

Reliability or risk of control policies is also an important metric when implementing system optimal routing in the field. In traffic operations, a question of high interest is how bad the worst cases during RL online process are compared with the User Equilibrium (i.e., no-control) scenario. On top of the well-known risk measure Conditional Value at Risk (CVaR) which was originally proposed in finance \cite{acerbi2002expected} and has been popular in safe RL recently \cite{chan2019measuring}, we propose a risk measure Conditional Travel Time Reduction at Risk (CTTRaR) for travel time reduction problems. CVaR for total travel time is given by
\begin{equation}
    \text{CVaR}_{x} (TTT) = \mathbb{E}[TTT \mid TTT \geq VaR_{x}(TTT)]
\end{equation}
where $VaR_x(TTT)$ is the $x$-quantile of the distribution of total travel time ($TTT$). Then, CTTRaR is
\begin{equation}
    \text{CTTRaR}_{x}(TTT) = \text{CVaR}_{x} (TTT_{UE}) - \text{CVaR}_{x} (TTT)
\end{equation}
where $\text{CVaR}_{x} (TTT_{UE})$ is the CVaR of $TTT$ in the UE (no-control) scenario. Generally, CTTRaR measures, in the worst performance cases, how much total travel time the control method can reduce compared with the no-control scenario. CTTRaR focuses on the left-most tail of the distribution.

\subsection{Approach 1: deterministic system optimal dynamic traffic assignment as a transportation domain model}
\label{sec: model-based method}
To solve the above real-time system optimal routing problems, one approach is to solve an optimal solution offline by leveraging traffic models and estimated demands. First, a traffic model, denoted as $\mathcal{\Tilde{M}}$, is built to simulate the real environment $\mathcal{M}$. The travel demands, denoted as $\mathbf{q}$, are estimated based on historical data and the estimated demands are denoted as $\mathbf{\Tilde{q}}$. Then, one can solve a system-optimal dynamic traffic assignment (SODTA) on top of the traffic model and the estimated demands. Eventually, the solved SODTA can suggest optimal path flows for each time interval. During the online experiments, the vehicles are assigned with paths according to the pre-calculated path ratios.

In this study, we implement the solution algorithm in \cite{qian2012system} to solve SODTA. First, the path marginal cost is approximated, so the system-optimal paths can be identified. Then, the well-known method of successive averages (MSA) can be applied to solve SODTA iteratively. The policy obtained by solving SODTA offline is denoted as pre-DSO.

The performance of pre-DSO depends on the difference between $\mathcal{\Tilde{M}}$ and $\mathcal{M}$ and the estimation error of $\mathbf{\Tilde{q}}$. On if the traffic model and the estimated demands perfectly align with the real case (i.e., $\mathcal{\Tilde{M}} = \mathcal{M}$ and $\mathbf{\Tilde{q}} = \mathbf{q}$), pre-DSO is the optimal solution to the routing problems. This is clearly not viable in practice. Thus, we must consider a model mismatch between $\mathcal{M}$ and $\mathcal{\Tilde{M}}$, and uncertain demands varying from estimated demand. As a result, pre-DSO is no longer guaranteed to be optimal, and possibly far from being optimal.

\subsection{Approach 2: model-free reinforcement learning}

The real-time optimal traffic routing problem can be formulated into an MDP. Then, the MDP can be solved by standard reinforcement learning with continuous action space (e.g., PPO \cite{schulman2017proximal} and SAC \cite{Haarnoja2018}). Here we introduce how the essential elements of the MDP are defined for the real-time routing problem.

\textbf{State}. States of the network or environment can be defined as
\begin{equation}
    s_t = [u_t^l], \forall l \in \mathcal{L}
\end{equation}
where $u_t^l$ is the number of vehicles passing link $l$ during time interval $t$. 

\textbf{Observation}. In realistic scenarios, the states of the whole road network cannot be accessed. Instead, only a portion of links is observed through sensors, and the set of observed links is $\mathcal{\hat{L}}$. More specifically, sensors collect real-time speeds and flows and send the averages during the last time interval at the beginning of the current time interval. Then, the control algorithm takes these real-time data (i.e., observations) as input and accordingly output actions. Therefore, observations of the environment are defined as
\begin{equation}
    o_t = [t, f_t^l, d_t^l], \forall l \in \mathcal{\hat{L}}
\end{equation}
where $f_t^l$ and $d_t^l$ are average flow and speed during time interval $t$ on link $l$. $t$ is included as an analogy of time of the day.

\textbf{Action}. In our formulation, the control algorithm determines the path assignment ratios, so actions are path ratios as follows.
\begin{equation}
    a_t = [\kappa_{i,t}^{ge}] , \forall p_i^{ge} \in \mathcal{P}^{ge}, g \in \mathcal{G}, e \in \mathcal{E}
\end{equation}

\textbf{Reward}. The control objective of the real-time system optimal routing problems is to minimize the total travel time of all vehicles during the whole time horizon as follows.
\begin{align}
    \min_{a_1, a_2, \cdots, a_t, \cdots} \eta \sum_t \sum_l u_t^l
\end{align}
where $\eta$ is the length of each time interval. If the reward $r_t = - \eta \sum_l u_t^l$ with $\gamma=1$, the control objective is equivalent to the objective of reinforcement learning as follows.

\begin{align}
   \min_{a_1, a_2, \cdots, a_t, \cdots} \eta \sum_t \sum_l u_t^l \iff \max _{a_1, a_2, \cdots, a_t, \cdots} \sum_t \gamma^t r_t
\end{align}
However, we do not assume $u_t^l$ is accessible for all links. An alternative and more practical way is to use the number of vehicles leaving the controlled road network during each time interval, which can be retrieved by monitoring the periphery of the controlled road network. The number of leaving/finished vehicles at the end of time interval $t$ is denoted as $F_t$ and the number of vehicles within the area at the beginning of the whole time horizon is denoted as $N_0$. Now, we have
\begin{align}
    \min_{a_1, a_2, \cdots, a_t, \cdots} & \eta \sum_t \sum_l u_t^l \\
    \iff  \min_{a_1, a_2, \cdots, a_t, \cdots} & \eta \sum_t \left( N_0 + \sum_{g} \sum_e q_t^{ge} - F_t \right) \\
    \iff \max_{a_1, a_2, \cdots, a_t, \cdots} & \sum_t F_t
\end{align}
Based on the above reasoning, the reward function is given as
\begin{equation}
    r_t \equiv \frac{F_t - \bar{F}_t}{N}
\end{equation}
where $\bar{F}_t$ is the average number of finished vehicles in the past and $N$ is a constant that approximately scales the reward into a range of [-10, 10]. We found the usage of $\bar{F}_t$ improved the training of reinforcement learning as it normalizes rewards along time intervals, and the usage of $N$ accelerated the training process as the reward scale is constrained into a small range and thereby quicker convergence in Q-values.

\subsection{Approach 3: TransRL}
\label{sec: TransRL for routing}

TransRL incorporates the traffic model $\Tilde{\mathcal{M}}$ in section \ref{sec: model-based method} into RL by developing a teacher policy based on the model-based policy pre-DSO, and then learns from the teacher policy. Specifically, pre-DSO can output an action $\tau(s)$ given a state $s$. We ignore the subscript $t$ here for readability.

$\tau(s)$ is not differentiable with regard to $s$ as the pre-DSO policy is solved iteratively using an intractable algorithm. This prohibits gradient propagation during training. Therefore, a differentiable neural network $\mu(s)$ is leveraged to imitate the pre-DSO policy by minimizing the mean squared loss between $\tau(s)$ and $\mu(s)$.

Then, the teacher policy $\pi_{TC}$ can be defined as a Multivariate normal distribution

\begin{equation}
    a_{TC} \sim  \pi_{TC}(\cdot \mid s) = \mathcal{N} (\mu(s), \Sigma)
\end{equation}
where the mean vector $\mu(s)$ is approximately equal to $\tau(s)$, and $\Sigma$ is assumed to be a diagonal matrix
\begin{equation}
\Sigma = 
\begin{bmatrix}
 \sigma & & & 0\\
 & \sigma\\
 & &  \ddots\\
 0 & & &  \sigma
\end{bmatrix}
\end{equation}

where $\sigma$ is the unreliability parameter to tune how much TransRL should rely on the model-based policy pre-DSO. Implicitly, as the unreliability parameter increases, TransRL is less dependent on the traffic model $\Tilde{\mathcal{M}}$. If the unreliability parameter is infinity, the action distribution is almost a uniform distribution, which means the information from the model is not used at all.

\section{Experiments}
\label{sec: experiments}
We conduct extensive experiments on three networks, including an abstractive two-route network, a synthetic network, and a large network with hundreds of links. In all three networks, we assume the demands are uncertain and follow Gaussian distributions $q_t^{qe} \sim \mathcal{N}(\Tilde{q}_t^{qe}, \beta \Tilde{q}_t^{qe})$, where $\Tilde{q}_t^{qe}$ is the historical average demand and $\beta$ is a parameter of demand uncertainty level. In the synthetic network, we assume a model mismatch between the traffic model $\Tilde{\mathcal{M}}$ and the real system dynamics $\mathcal{M}$, and the model mismatch stems from the parameter estimation errors (i.e., inaccurate free-flow speeds and critical densities). In the large network, the model mismatch is caused by unexpected incidents that block lanes. 

Our proposed algorithm TransRL is compared with various routing strategies and baselines, which are summarized as follows.
\begin{itemize}
    \item Model-based method: \textbf{pre-DSO}. This is the policy obtained in section \ref{sec: model-based method}, namely through  System Optimal Dynamic Traffic Assignment. 
    \item Model-free reinforcement learning: \textbf{PPO} \cite{schulman2017proximal}. PPO is a state-of-the-art policy optimization algorithm. PPO updates policy by maximizing the advantages while using a clip operation to make sure the updated policy does not go far from the old policy.
    \item Model-free reinforcement learning: \textbf{SAC} \cite{Haarnoja2018}. SAC is a state-of-the-art model-free reinforcement learning algorithm. SAC uses value networks to evaluate state-action pairs and policy networks to output actions. The most notable attribute of SAC is the entropy term in the objective. SAC aims at maximizing returns while maintaining a policy as stochastic as possible, which enables SAC to explore the environment sufficiently and avoid being stuck at a local optimum.
    \item No-control baseline: User Equibrlium\textbf{(UE)}. In the large network where some count and speed data are available, we conducted an dynamic OD estimation, and the path ratios from the OD estimation were applied and regarded as the UE baseline.
\end{itemize}

\subsection{An abstractive network with uncertain demands}
The first network is the abstractive 2-route network in \cite{pi2017stochastic}. The 2-route network is an abstraction of a general network during morning peak hours (see Figure \ref{subfig: 2-route net}). After abstraction, the network connects the residential neighborhoods (the origin) and the downtown (the destination) through two routes: (1) route 1 is the main corridor with high free-flow speed but with limited capacity, and (2) route 2 is the aggregation of all local streets with low free-flow speed but with sufficient capacity. The action is the ratio of demands choosing route 1.

During morning peak hours, typical demands increase and then decrease as Figure \ref{subfig: 2-route net demands} shows. At each time step, the demand follows a Gaussian distribution $q_t \sim \mathcal{N} (\mu_t, \beta \mu_t)$, where $\mu_t$ is the historical mean demand at time $t$ and $\beta$ is the demand uncertainty parameter. In our experiments, we considered three scenarios, including a low demand uncertainty scenario ($\beta = 0.05$), a medium demand uncertainty scenario ($\beta = 0.10$), and a high demand uncertainty scenario ($\beta = 0.20$).

\begin{figure}[H]
     \centering
     \begin{subfigure}[b]{0.45\textwidth}
         \centering
         \includegraphics[width=\textwidth]{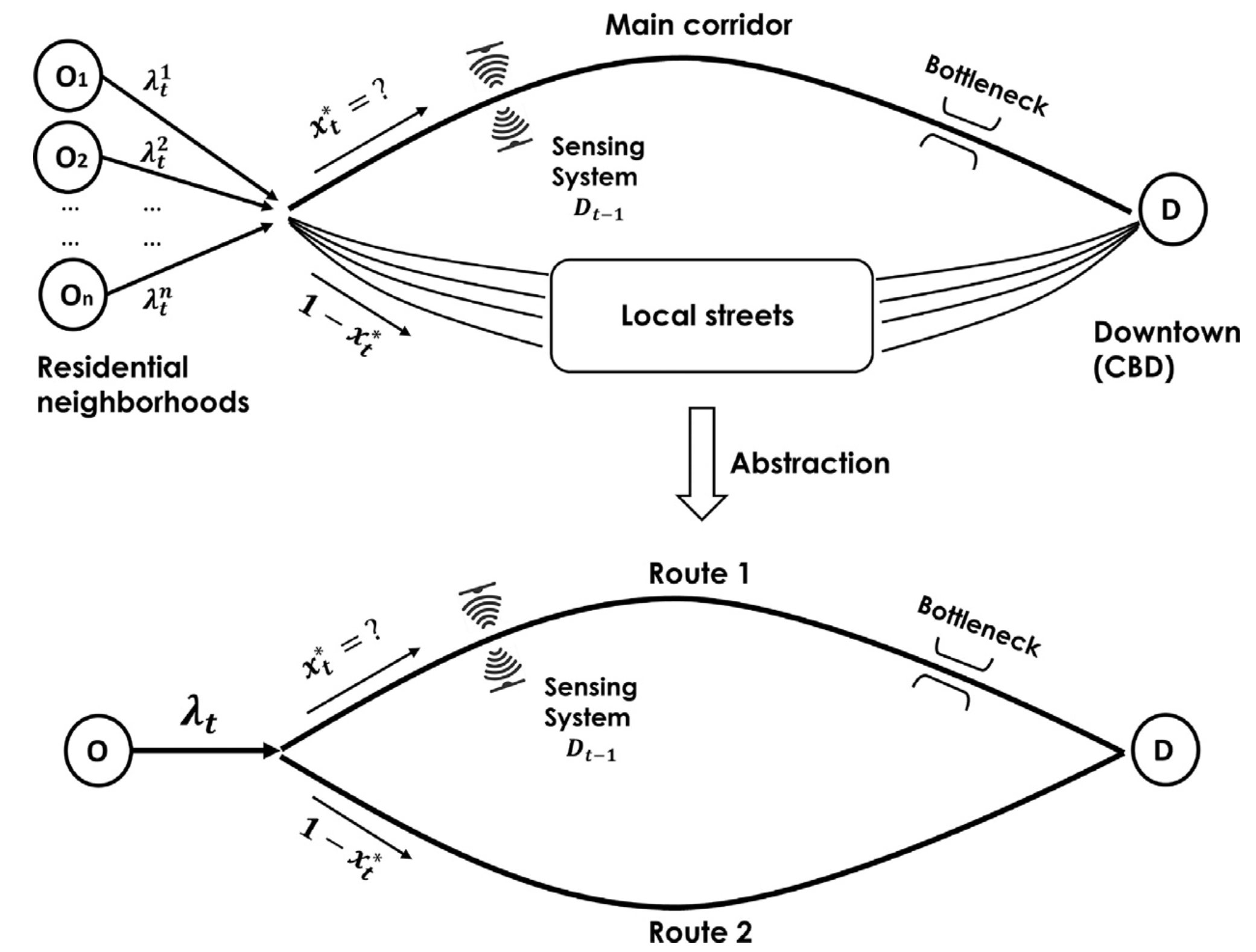}
         \caption{The abstractive 2-route network in \cite{pi2017stochastic}}
         \label{subfig: 2-route net}
     \end{subfigure}
     \begin{subfigure}[b]{0.45\textwidth}
         \centering
         \includegraphics[width=\textwidth]{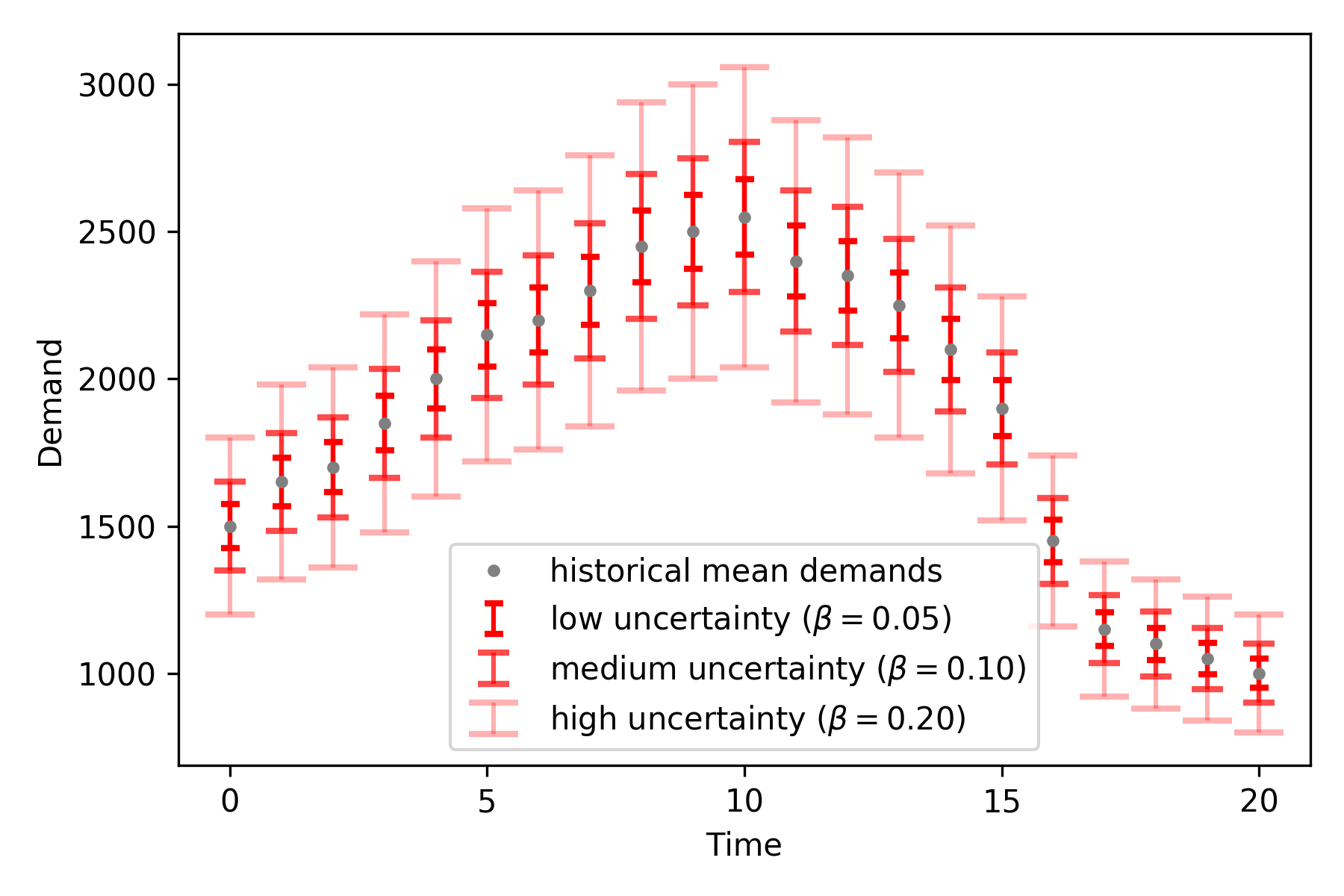}
         \caption{Demand distributions with different uncertainties}
         \label{subfig: 2-route net demands}
     \end{subfigure}
        \caption{The abstractive network and demands}
        \label{fig: 2-route network}
\end{figure}

\subsubsection{Training of reinforcement learning}
TransRL, PPO, and SAC were trained on all three scenarios with the same observation definition, action definition, and reward function. For comparison, the hyperparameters of each RL are the same in three scenarios. The training process is plotted in Figure \ref{fig: training on the 2-route network}. There are two main findings. First, among all three scenarios, TransRL converged quicker than PPO and SAC, while even after convergence, returns of TransRL were higher than PPO and SAC. This indicates that TransRL is more efficient in finding a good policy and this policy is better than those found by PPO and SAC. In addition, the differences between the three RL decrease as the demand uncertainty level increases. This is reasonable. As the demand uncertainty level increases, the actions from the pre-DSO are less close to the optimal actions because the demand estimation is less accurate. As a result, the information from the teacher policy is less informative when the demand uncertainty level is larger.

\begin{figure}[H]
     \centering
     \begin{subfigure}[b]{0.49\textwidth}
         \centering
         \includegraphics[width=\textwidth]{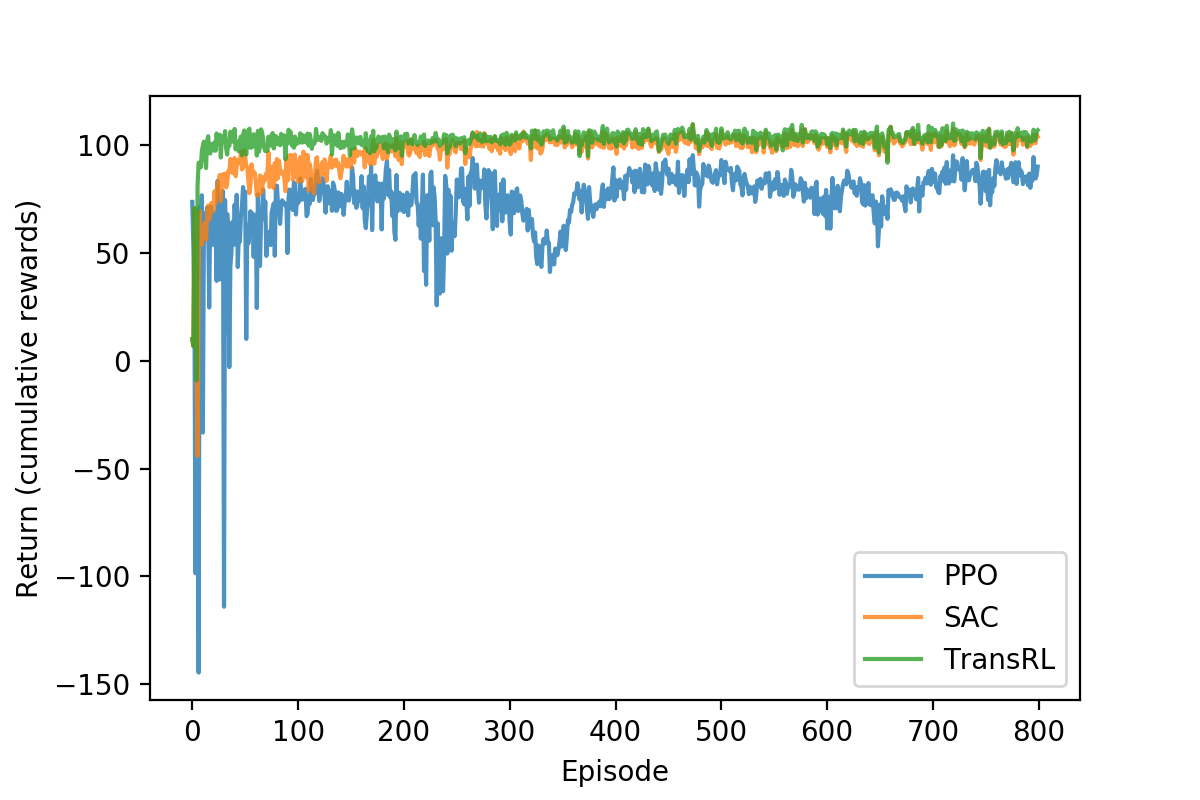}
         \caption{Low demand uncertainty}
     \end{subfigure}
     \begin{subfigure}[b]{0.49\textwidth}
         \centering
         \includegraphics[width=\textwidth]{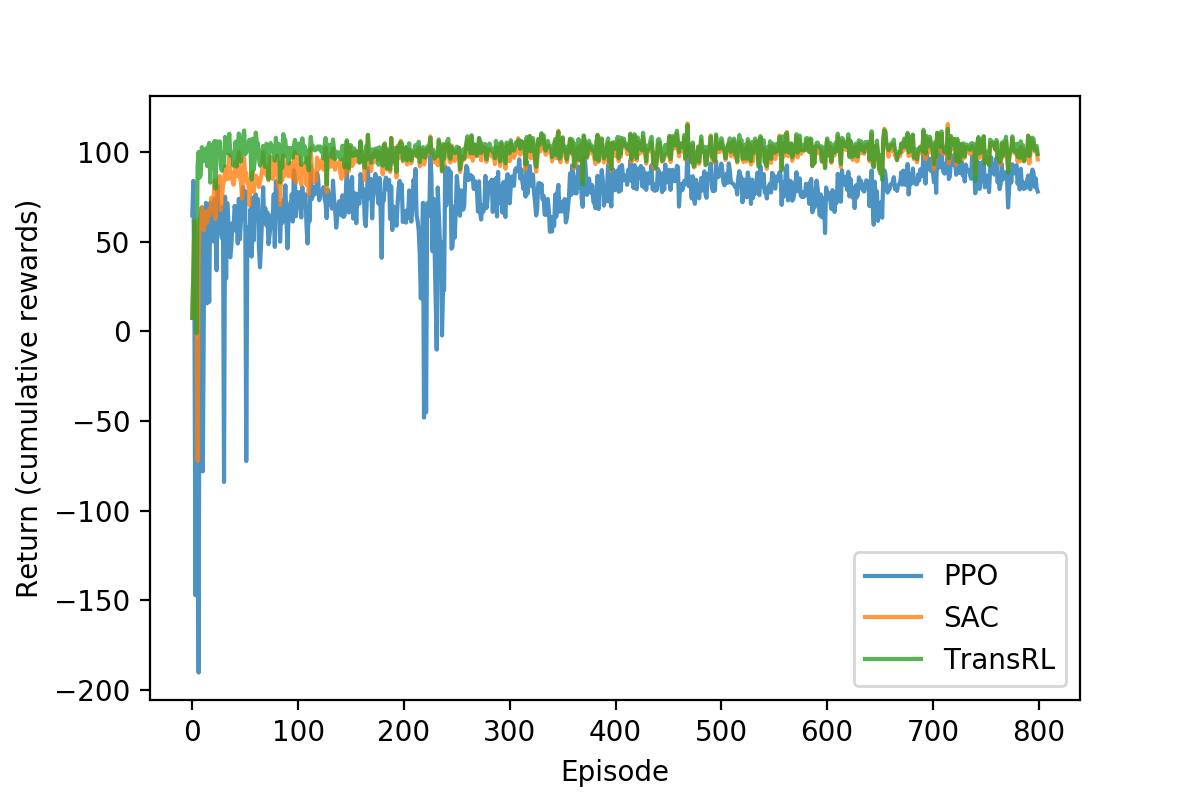}
         \caption{Medium demand uncertainty}
     \end{subfigure}
    \begin{subfigure}[b]{0.49\textwidth}
         \centering
         \includegraphics[width=\textwidth]{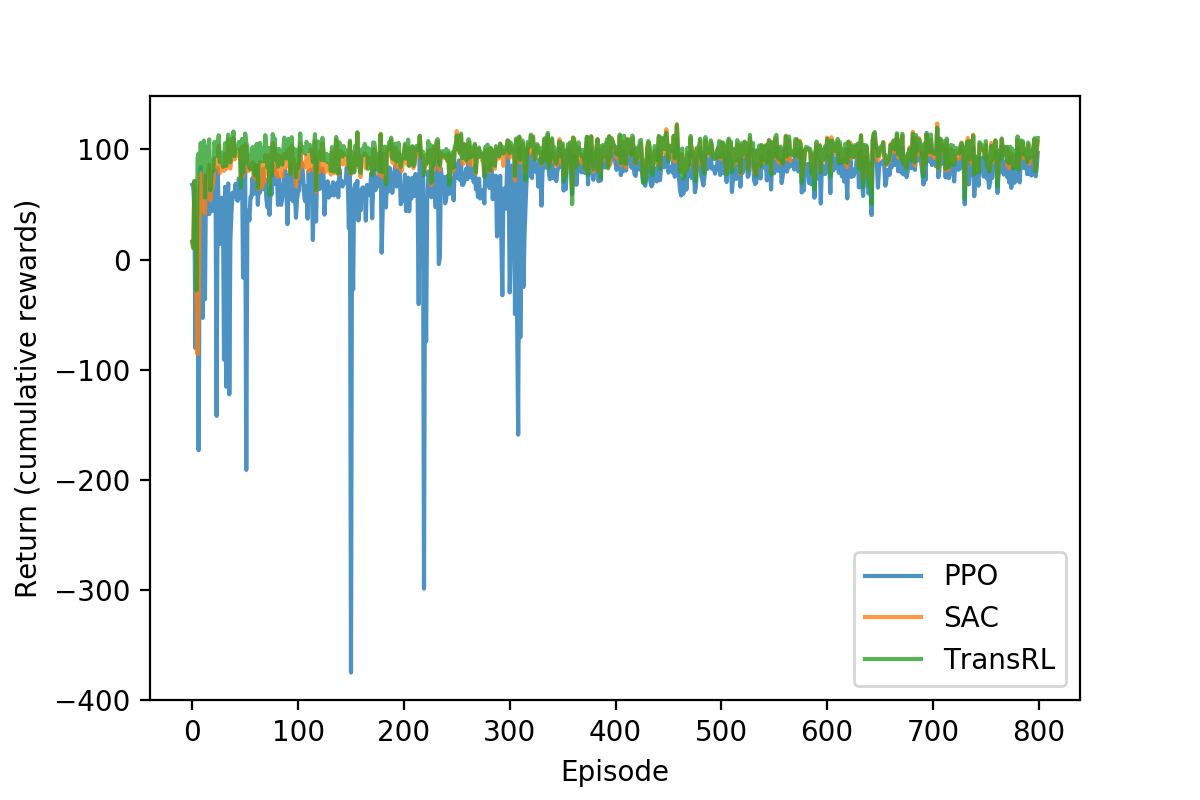}
         \caption{High demand uncertainty}
     \end{subfigure}
        \caption{Training on the 2-route network}
        \label{fig: training on the 2-route network}
\end{figure}

To shed insights on how TransRL learns from the teacher policy, the actions from TransRL and SAC are plotted in Figure \ref{fig: policy evolution on the 2-route network}. For comparison, the actions from the pre-DSO and DSO are also included. Given the single-peak demands, the optimal path ratio for route 1 should decrease during the first half to deviate vehicles to route 2, and then increase to assign more vehicles back to route 1 in general \cite{pi2017stochastic}. The optimal path ratios depend on the specific demands. While DSO uses the actual demands to calculate the optimal path ratios, pre-DSO calculates the "optimal" path ratios using historical mean demands which work as an approximation of the actual demands. As pre-DSO depends on historical mean demands, the actions of pre-DSO are pre-calculated and deterministic along different episodes and different demand scenarios. DSO uses actual demands, so the actions are calculated in real time. As Figure \ref{fig: policy evolution on the 2-route network} shows, the difference between pre-DSO and DSO increases as demand uncertainty becomes high because the actual demands deviate more from the historical mean demands in the high demand uncertainty scenario. As a consequence, pre-DSO performs worse as demand uncertainty increases.

Compared with SAC, TransRL is much more efficient in retrieving a policy close to the optimal policy. In all three demand scenarios, it took 100 episodes that SAC learned to decrease the action during the first 10 time steps, and it took another 300 episodes for SAC to learn to increase the actions during the last 10 time steps. In contrast, TransRL learned the general pattern of decreasing and then increasing actions after just 5 episodes, which indicates the teacher policy boosted the learning process of TransRL. Furthermore, the actions from TransRL are more stable than SAC, which indicates the policy from TransRL is more interpretable. Notably, different from pre-DSO whose actions are fixed, TransRL is adaptive to real-time states and can continuously learn from interactions.

\begin{figure}[H]
     \centering
     \begin{subfigure}[b]{\textwidth}
         \centering
         \includegraphics[width=\textwidth]{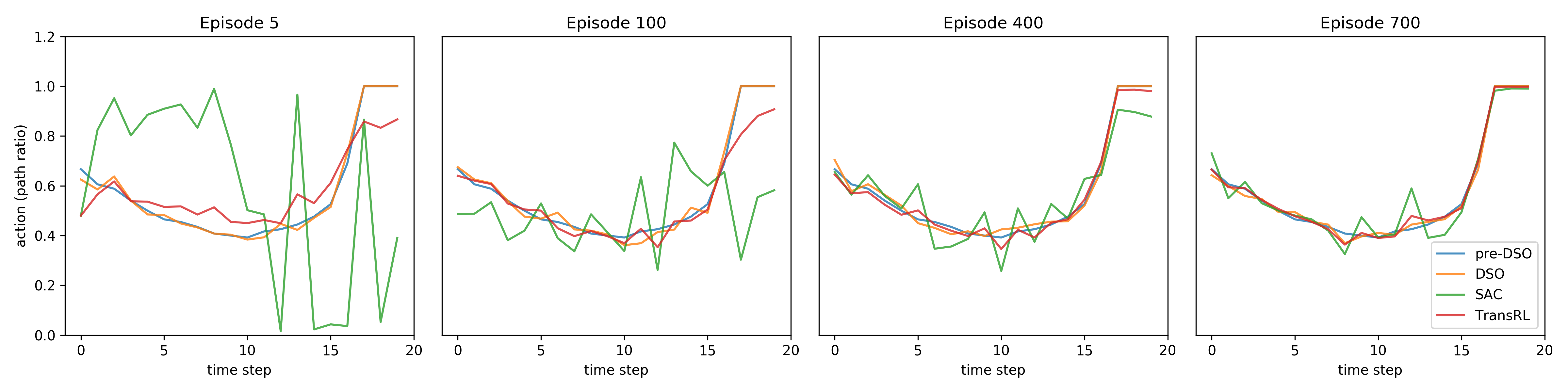}
         \caption{Low demand uncertainty}
     \end{subfigure}
     \begin{subfigure}[b]{\textwidth}
         \centering
         \includegraphics[width=\textwidth]{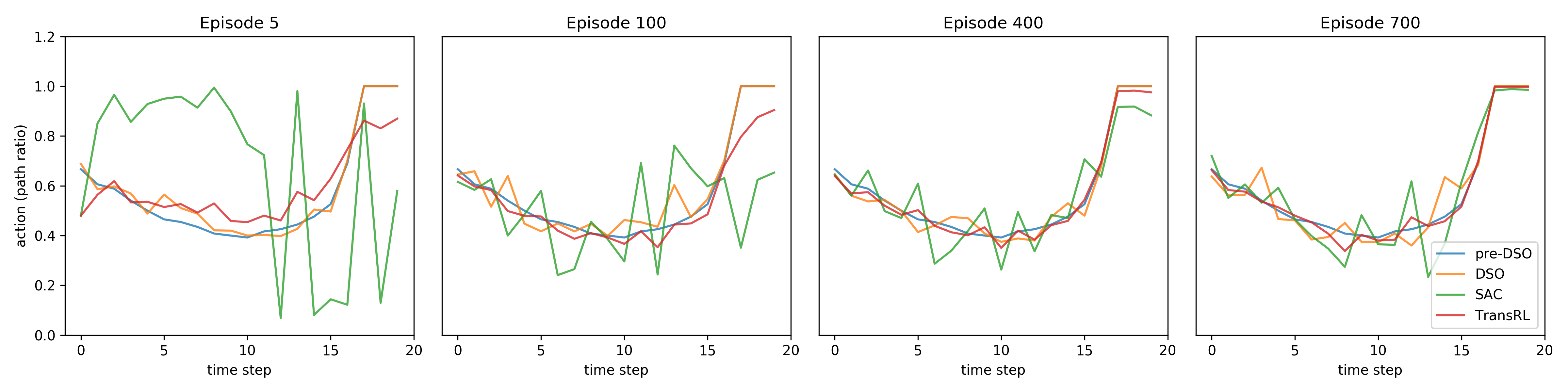}
         \caption{Medium demand uncertainty}
     \end{subfigure}
    \begin{subfigure}[b]{\textwidth}
         \centering
         \includegraphics[width=\textwidth]{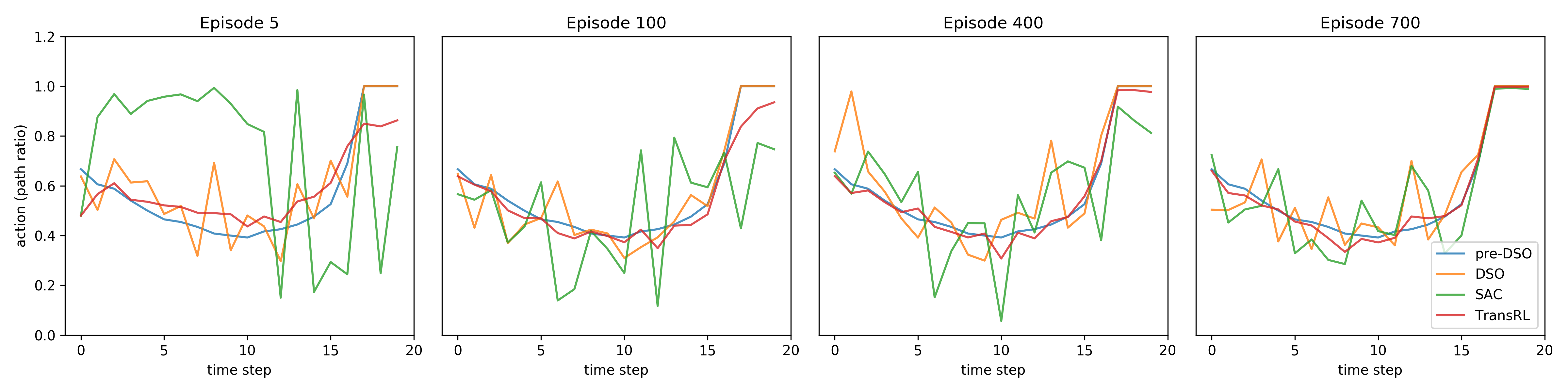}
         \caption{High demand uncertainty}
     \end{subfigure}
        \caption{Policy evolution on the 2-route network}
        \label{fig: policy evolution on the 2-route network}
\end{figure}

\subsubsection{Control performances of online tests}
After the training, we tested all control methods for 100 episodes under uncertain demands with the same random seed. In the no-control case (i.e., UE), the travel time on both routes is identical so that no traveler can benefit from changing routes. For comparison, we also tested the stochastic optimal method in \cite{pi2017stochastic}, which can be regarded as the optimum without knowing the actual demands in advance. \cite{pi2017stochastic} analytically formulates the optimal routing problem of the 2-route network and solves the stochastic optimal policy using dynamic programming, denoted as DP. 

The total travel time under different control methods is summarized in Table \ref{tab: TTT of 2-route net} and plotted in Figure \ref{fig: TTT on 2-route}. Compared with the no-control case, all control methods could significantly reduce the total travel time. The comparison between the model-based method (i.e., pre-DSO) and the model-free RL (i.e., PPO and SAC) is interesting. With low or medium demand uncertainty, pre-DSO led to lower total travel time than PPO and SAC. However, when the demand uncertainty is high, SAC outperformed pre-DSO. This indicates the model-based method performs well when the actual system is close to the approximation, but it is brittle to information inaccuracy. While model-free RL is adaptive, it does not fully utilize prior information.

Notably, in all three demand scenarios, TransRL outperformed both the model-based method and the model-free RL in reducing total travel time and was pretty close to the stochastic optimum (i.e., DP). This can be attributed to TransRL combining the advantages of RL and the model-based method. As an RL, TransRL is able to adapt to real-time states and learn from the training experiences. On the other hand, it utilizes the prior information from the model and historical mean demands.

The reliability of control methods is measured using the proposed metric $\text{CTTRaR}_x$, which is listed in Table \ref{tab: CTTRaR of 2-route net}. Higher $\text{CTTRaR}_x$ indicates more reliability in worst cases. With low demand uncertainty, pre-DSO is more reliable than SAC, but SAC is more reliable than pre-DSO when demand uncertainty is medium or high. Except for the stochastic optimum (DP), TransRL is the most reliable with low or medium demand uncertainty, and SAC is the most reliable with high demand uncertainty. This indicates that when demand uncertainty is high, a more stochastic policy is more reliable. Overall, TransRL is more reliable than pre-DSO, PPO, and SAC.

\begin{table}[H]
\centering
\caption{Total travel time on the 2-route network during 100 test episodes}
\label{tab: TTT of 2-route net}
\begin{tabular}{@{}lp{1.5cm}p{1cm}p{1.5cm}p{1cm}p{1.5cm}p{1cm}@{}}
\toprule
   & \multicolumn{2}{c}{\makecell{Low demand \\ uncertainty}} & \multicolumn{2}{c}{\makecell{Medium demand \\ uncertainty}} & \multicolumn{2}{c}{\makecell{High demand \\ uncertainty}} \\ \midrule
Total travel time & Average        & SD       & Average          & SD         & Average         & SD        \\ \midrule
UE (no control)      & 17334.72         & 266.92 & 17346.28         & 522.11 & 17341.98          & 997.96 \\
pre-DSO (model-based) & 9592.76          & 350.95 & 9959.54          & 706.23 & 10683.08          & 1415.86 \\
PPO (model-free RL)     & 11410.16         & 450.78 & 11193.19         & 633.29 & 11597.74          & 1091.86 \\
SAC (model-free RL)    & 9782.19          & 239.96 & 9971.37          & 488.80 & 10463.12          & 930.34  \\
TransRL (ours) & \underline{9592.45} & 255.91 & \underline{9817.92} & 496.42 & \underline{10357.36}    & 1031.81  \\ 
DP (stochastic optimal)     & \textbf{9450.98}    & 239.24 & \textbf{9730.62}    & 497.39 & \textbf{10267.50} & 918.97  \\ \bottomrule
\end{tabular}
\end{table}

\begin{figure}[H]
    \centering
    \includegraphics[width=0.6\textwidth]{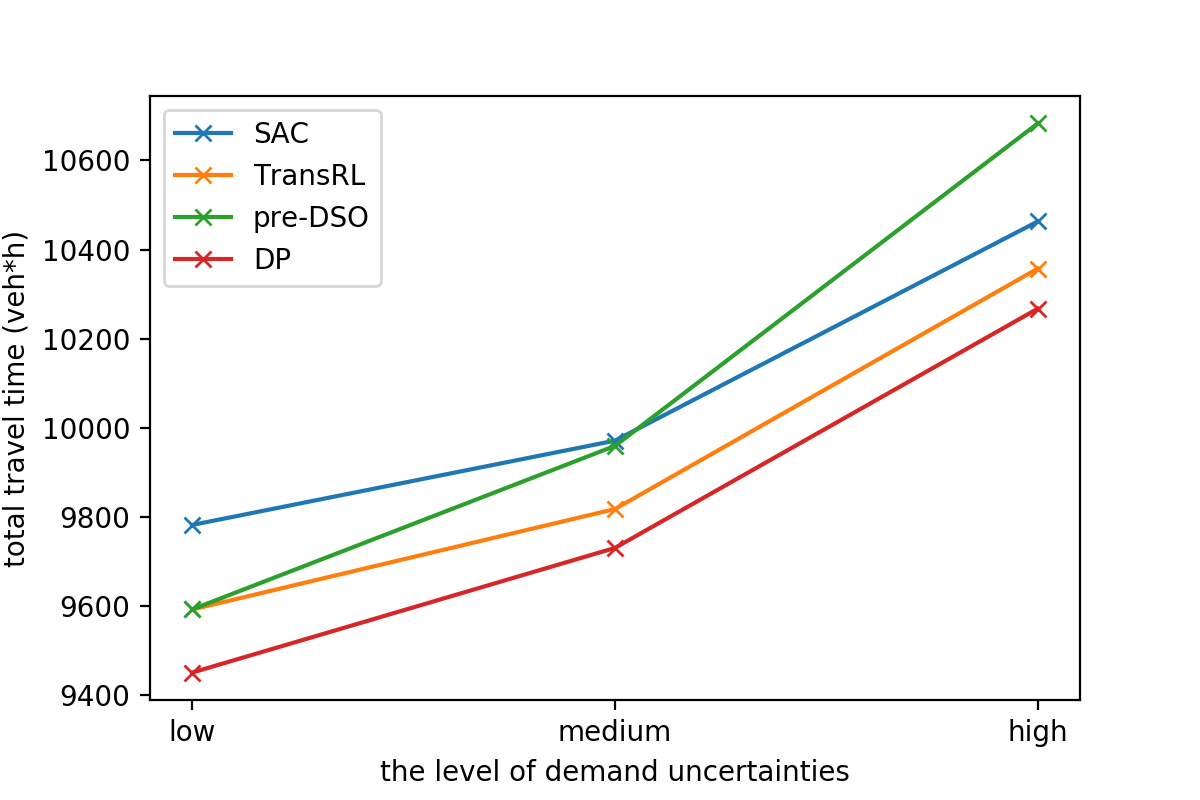}
    \caption{Average total travel time on the 2-route network during 100 test episodes}
    \label{fig: TTT on 2-route}
\end{figure}

\begin{table}[H]
\centering
\caption{$\text{CTTRaR}_{x}$ on the 2-route network during 100 test episodes. The quantile $x$ is set to be 10\%, which means the top 10\% worst episodes.}
\label{tab: CTTRaR of 2-route net}
\begin{tabular}{@{}llll@{}}
\toprule
   &\makecell{Low demand \\ uncertainty} &\makecell{Medium demand \\ uncertainty} & \makecell{High demand \\ uncertainty} \\ \midrule
   & $\text{CTTRaR}_{x}$ &$\text{CTTRaR}_{x}$ & $\text{CTTRaR}_{x}$ \\ \midrule
pre-DSO (model-based) & 7510.85 &	6881.29	& 5535.49
     \\
PPO (model-free RL)     & 5318.24 &	5531.50 & 6068.37
   \\
SAC (model-free RL)    & 7507.93 &	7385.84 &	\underline{7054.28}
        \\
TransRL (ours) & \underline{7700.58}	& \underline{7422.13} &	6764.85
   \\ 
DP (stochastic optimal)     &  \textbf{7917.94} & \textbf{7580.57} &	\textbf{7180.80}
  \\ \bottomrule
\end{tabular}
\end{table}

\subsubsection{Sensitivity analysis of the unreliability parameter}
In TransRL, the unreliability parameter $\sigma$ determines how much the TransRL should depend on the transportation method and how extensive TransRL should explore the action space. To look into the impacts of the unreliability parameter, we compared the training processes and the control performances of TransRL with different unreliability parameters.

The training processes with different unreliability parameters under three demand scenarios are plotted in Figure \ref{fig: SA on the 2-route network}. In general, the training process was insensitive to the unreliability parameter. After convergence, different parameters led to similar performances. In low or medium demand uncertainty scenarios, the value of the unreliability parameter did affect the convergence speed, and smaller unreliability parameters led to faster convergence. While in the high demand uncertainty, the convergence speed was almost identical along different parameters. The control performances with different parameters align with this phenomenon. As Figure \ref{fig: SA TTT on the 2-route network} shows, the control performances were generally incentive to the parameters, and small unreliability parameters performed better than large unreliability parameters in low or medium demand uncertainty scenarios.

\begin{figure}[H]
     \centering
     \begin{subfigure}[b]{0.49\textwidth}
         \centering
         \includegraphics[width=\textwidth]{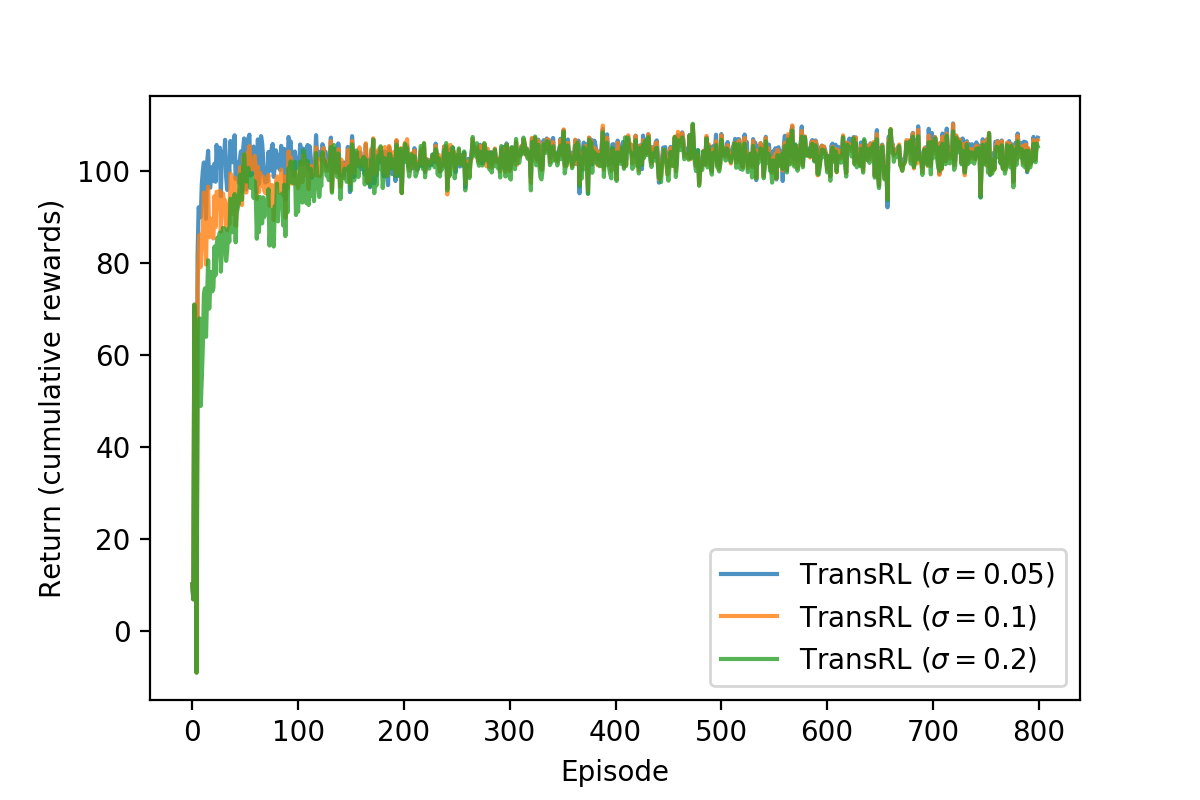}
         \caption{Low demand uncertainty}
     \end{subfigure}
     \begin{subfigure}[b]{0.49\textwidth}
         \centering
         \includegraphics[width=\textwidth]{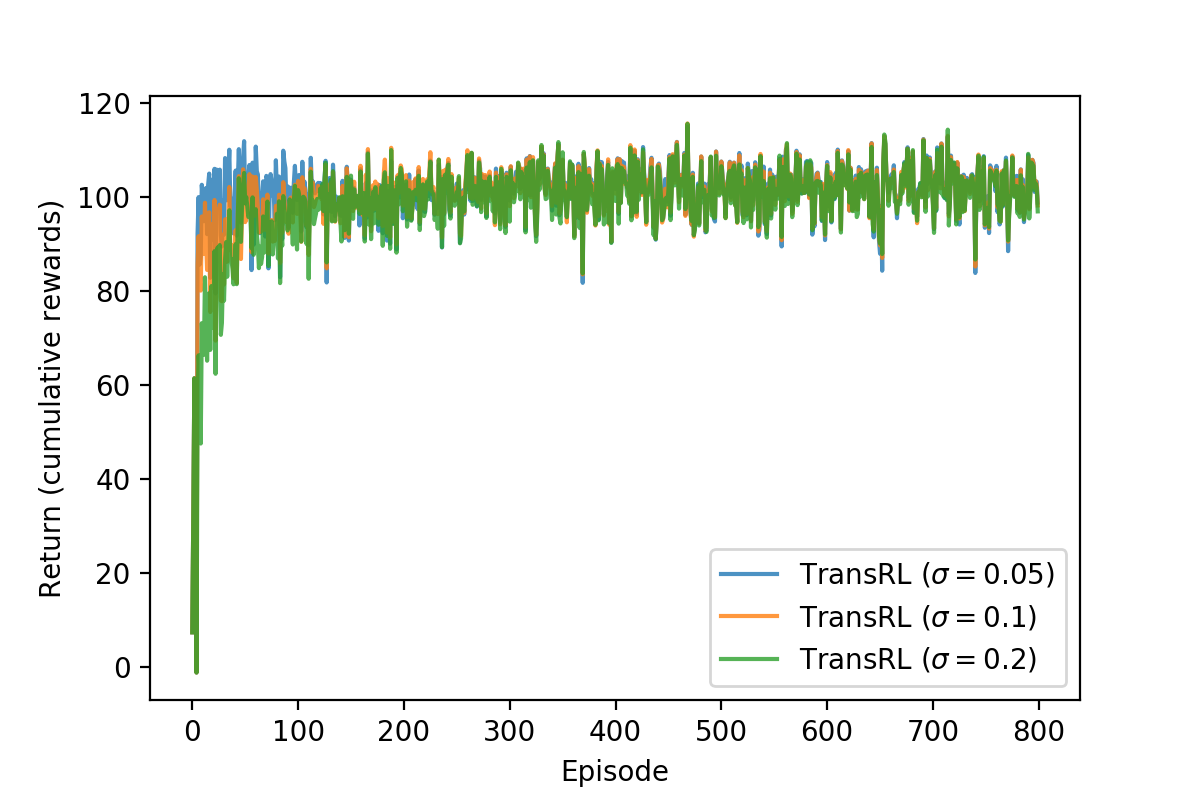}
         \caption{Medium demand uncertainty}
     \end{subfigure}
    \begin{subfigure}[b]{0.49\textwidth}
         \centering
         \includegraphics[width=\textwidth]{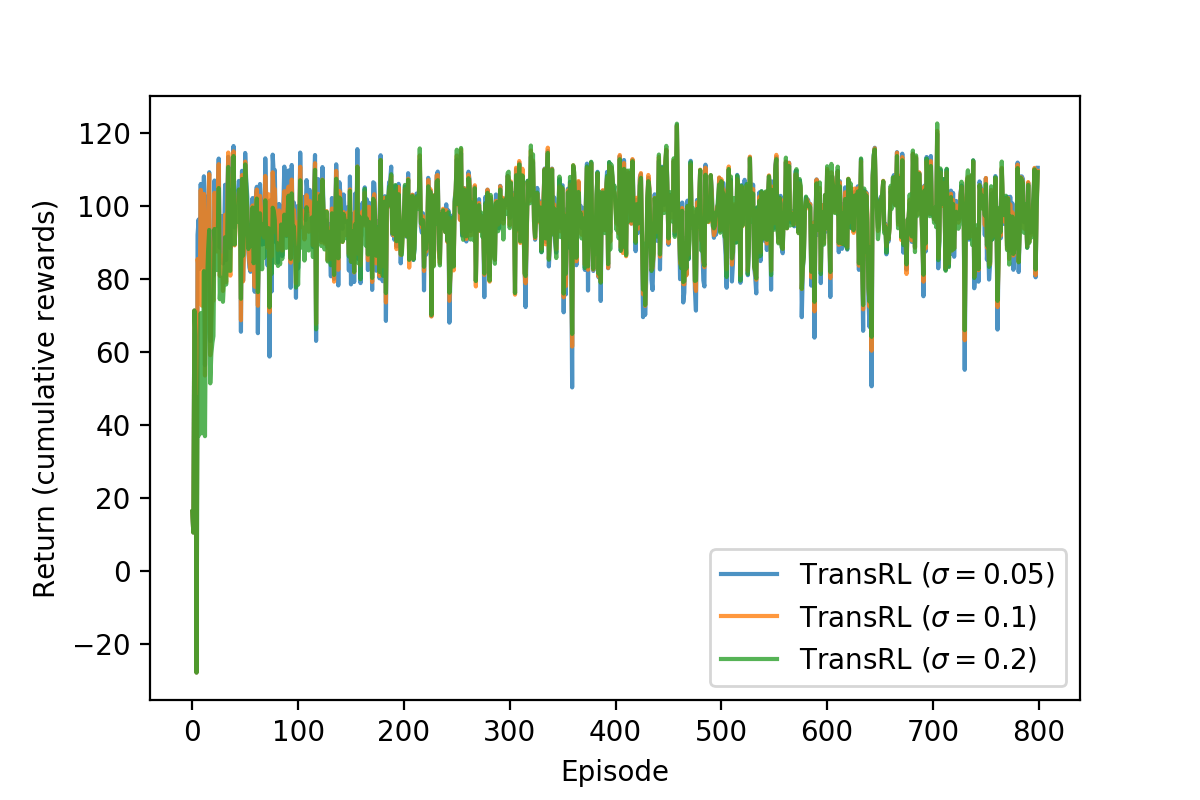}
         \caption{High demand uncertainty}
     \end{subfigure}
        \caption{Sensitivity analysis on the 2-route network: training process}
        \label{fig: SA on the 2-route network}
\end{figure}

\begin{figure}[H]
     \centering
     \begin{subfigure}[b]{0.49\textwidth}
         \centering
         \includegraphics[width=\textwidth]{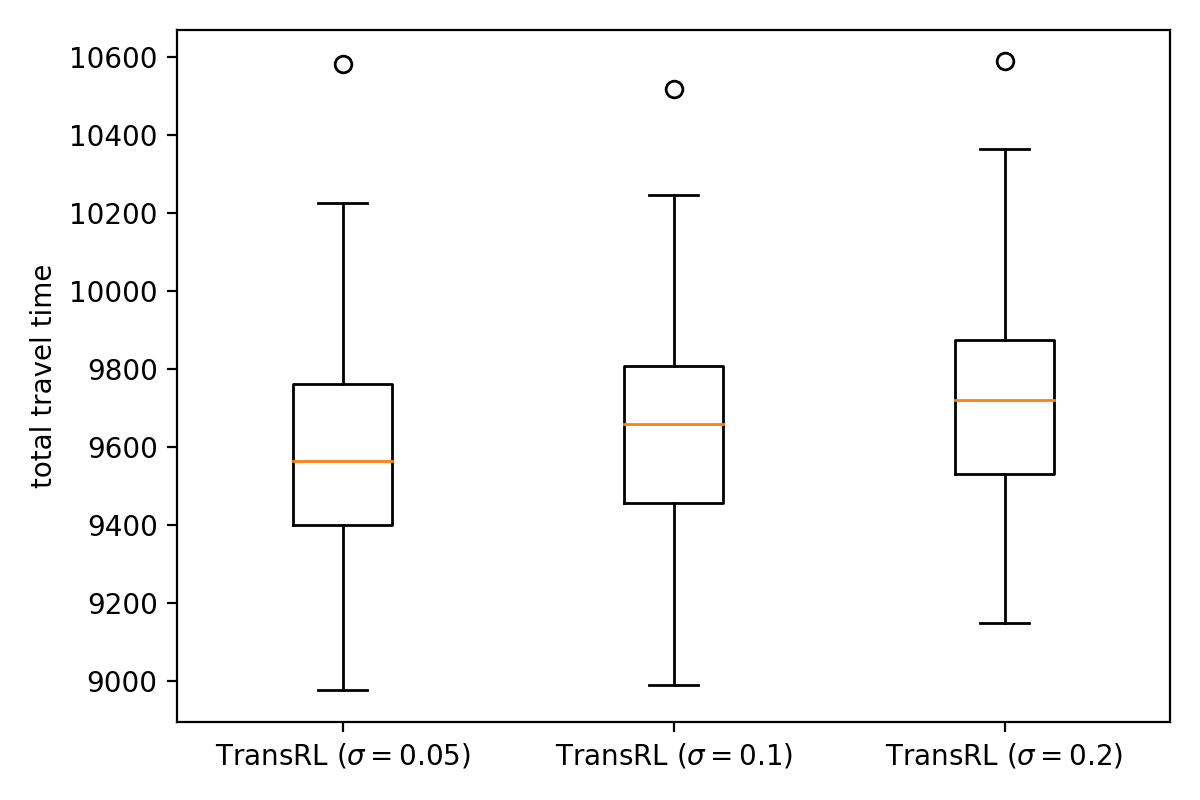}
         \caption{Low demand uncertainty}
     \end{subfigure}
     \begin{subfigure}[b]{0.49\textwidth}
         \centering
         \includegraphics[width=\textwidth]{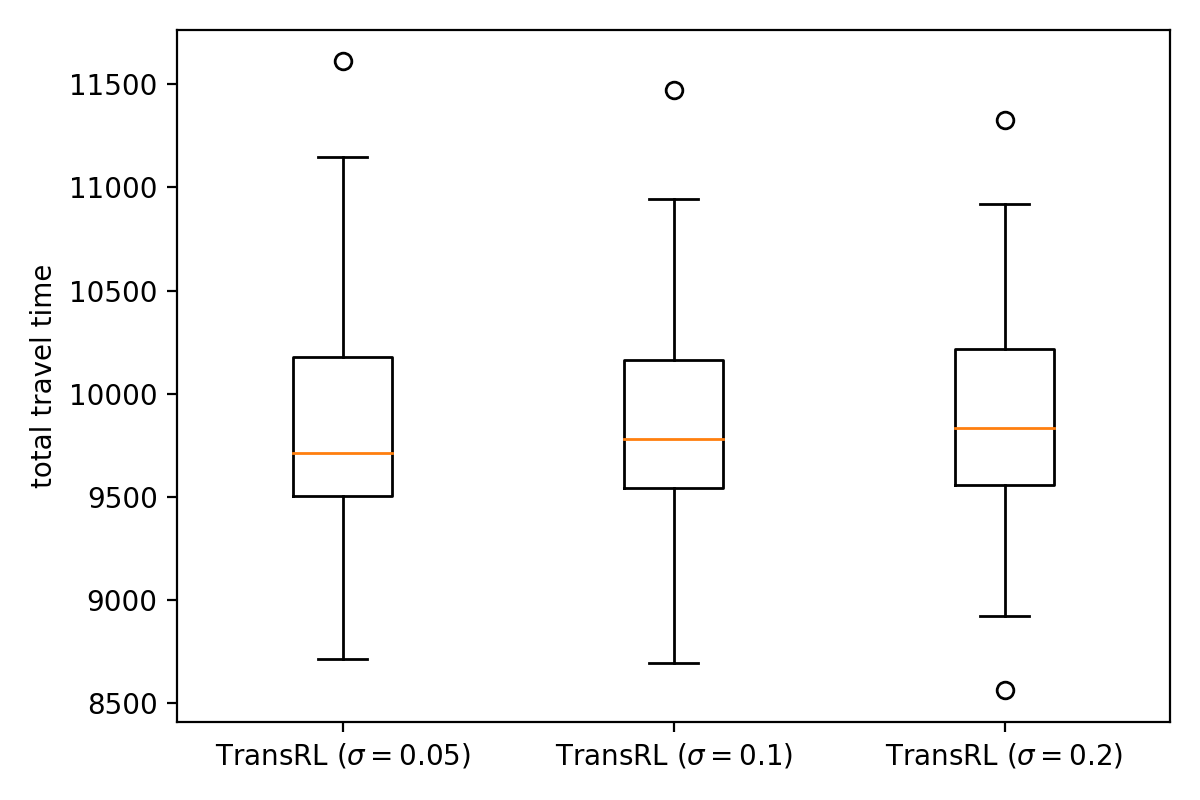}
         \caption{Medium demand uncertainty}
     \end{subfigure}
    \begin{subfigure}[b]{0.49\textwidth}
         \centering
         \includegraphics[width=\textwidth]{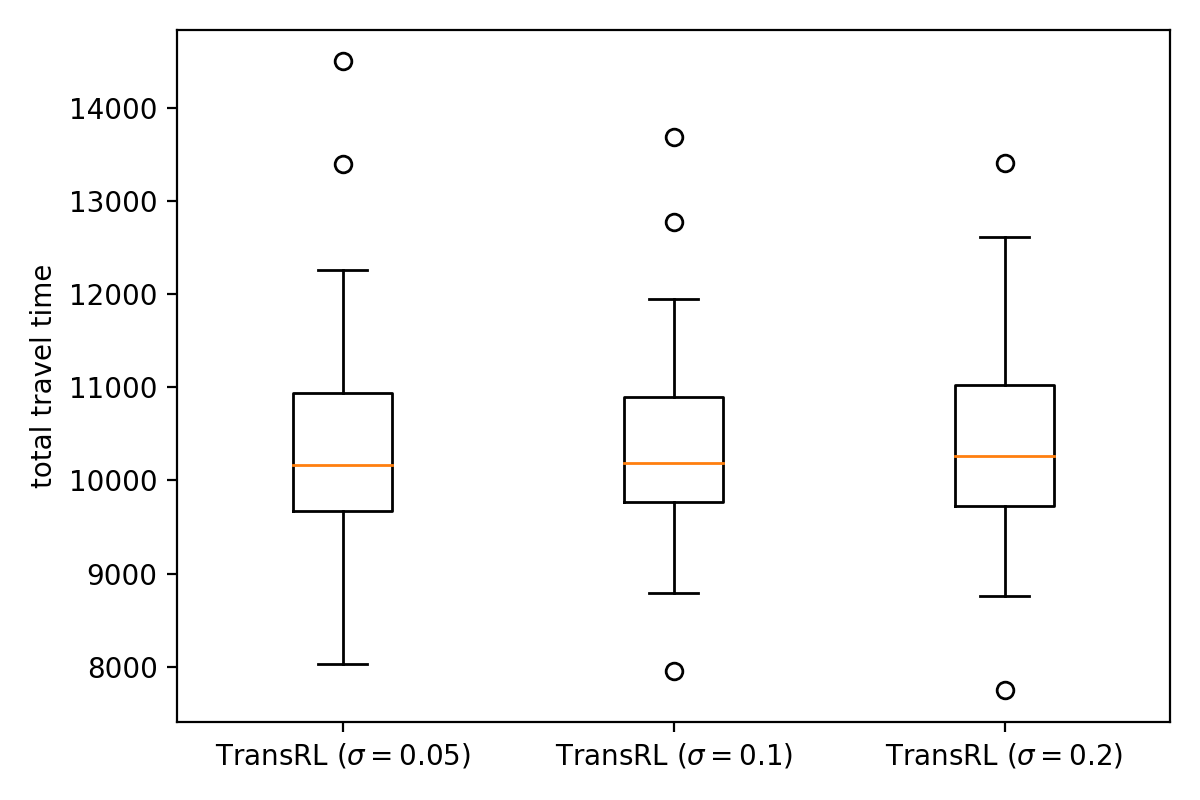}
         \caption{High demand uncertainty}
     \end{subfigure}
        \caption{Sensitivity analysis on the 2-route network: total travel time}
        \label{fig: SA TTT on the 2-route network}
\end{figure}

\subsection{A synthetic network with uncertain demands and model mismatch}

The second network for experiments is a synthetic corridor network first used in \cite{nie2006variational}. As Figure \ref{fig: corridor network} shows, the corridor network, which aims to abstract a commuting network, is composed of 18 links and 16 nodes. Links 1-2, 2-4, 4-6, 6-8, and 8-10 are freeway links, so these links has high free-flow slow and large capacity. In addtion, link 6-8 is a bottleneck with smaller capacity caused by a lane-drop. Links 3-5, 5-7, and 7-9 are arterial links with lower free-flow speed than freeway links. Links 3-2, 4-5, 7-6, and 8-9 are on-ramp or off-ramp connecting freeway links with arterial links. All the other links are local streets connecting nodes with arterial links or freeway links.

We consider many-to-one OD pairs and each OD pair has a typical single-peak demand pattern. Specifically, node 11 and node 12 are the origin nodes, while node 15 is the destination node. At each time step, the demand follows a Gaussian distribution $q_t^{ge} \sim \mathcal{N} (\mu_t^{ge}, \beta \mu_t^{ge})$, where $\mu_t^{ge}$ is the historical mean demand of OD pair $ge$ at time $t$ and $\beta=0.10$ is the demand uncertainty parameter. Between each OD pair, there are multiple routes composed of freeway links and/or arterial links to choose from.

In this network, we also consider a model mismatch between the actual system dynamics ($\mathcal{M}$) that is unknown in practice, and the approximated/estimated traffic model ($\Tilde{\mathcal{M}}$). The network is modeled by CTM, and CTM cells are defined by length, the number of lanes, free-flow speed, capacity, and jam density. In the real world, the estimated free-flow speed and capacity may differ from the actual unknown values, which causes a model mismatch. In our experiments, we randomly set the free-flow speed and capacity in the traffic model $\Tilde{\mathcal{M}}$ to be 10-20\% higher or lower than those in the actual system dynamics $\mathcal{M}$ where the online RL models are not offline trained by.

This scenario is more challenging than the 2-route network. First, there are more links and routes than the 2-route network, so the state space and action space are larger. Besides, some routes share the same links, which makes the system dynamics more complicated. More importantly, a model mismatch exists, so it's more challenging for the model-based method (pre-DSO) and TransRL.

\begin{figure}[H]
    \centering
    \includegraphics[width=0.6\textwidth]{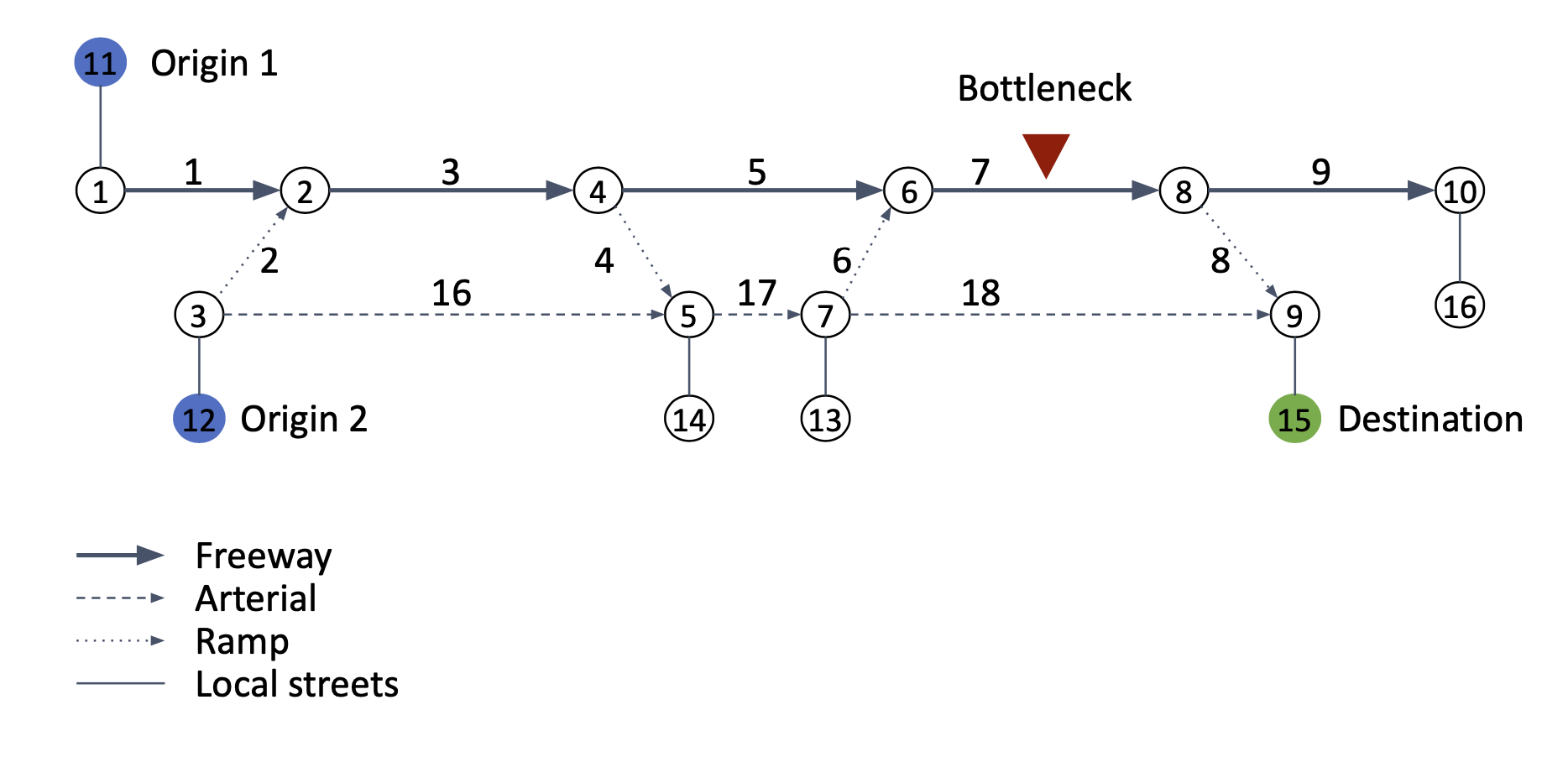}
    \caption{The synthetic corridor network}
    \label{fig: corridor network}
\end{figure}

\subsubsection{Training of reinforcement learning}
The training processes of PPO and SAC are plotted in Figure \ref{fig: training on corridor network}. PPO was struggling to find a reasonable policy and did not converge after 2000 episodes. While SAC converged after approximately 500 episodes, the performance was pretty fluctuating even after the convergence. The performance of PPO and SAC indicates the increase in state and action space makes model-free reinforcement learning less efficient than learning from random exploration.

TransRL with different unreliability parameters is also compared in Figure \ref{fig: training on corridor network}. With various unreliability parameters, TransRL learned faster and reached a more stable performance than PPO and SAC. Besides, comparing TransRL with $\sigma=0.05$ and TransRL with $\sigma=0.10$, we can find TransRL with $\sigma=0.05$ converged faster than TransRL with $\sigma=0.10$, but TransRL with $\sigma=0.10$ ended with a slightly higher return than TransRL with $\sigma=0.10$. This indicates a trade-off between the convergence speed and the final performance. When $\sigma=0.20$, the TransRL was similar to SAC, which supported Theorem \ref{theorem: SAC and TransRL} claiming SAC is equivalent to TransRL with a teacher policy of uniform distribution (i.e., $\sigma = \infty$ in this case).

\begin{figure}[H]
    \centering
     \begin{subfigure}[b]{0.49\textwidth}
         \centering
         \includegraphics[width=\textwidth]{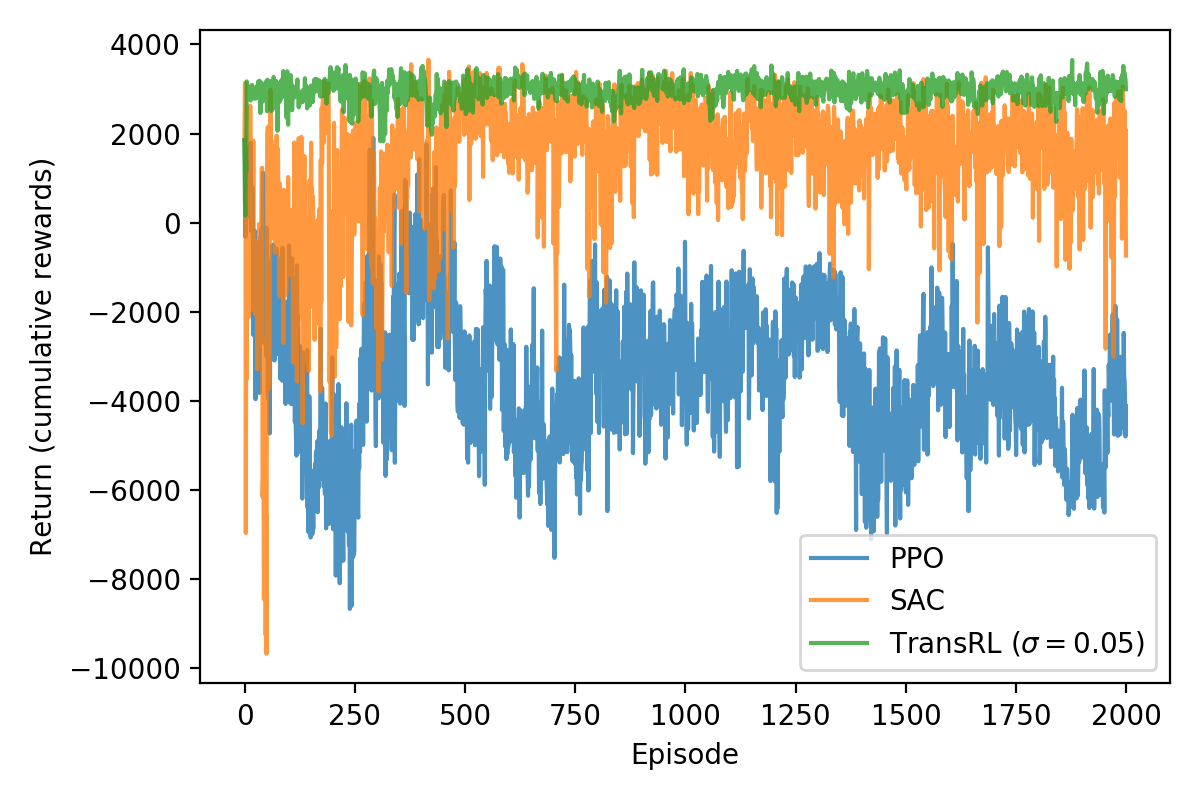}
         \caption{TransRL with an unreliability parameter of 0.05}
     \end{subfigure}
     \begin{subfigure}[b]{0.49\textwidth}
         \centering
         \includegraphics[width=\textwidth]{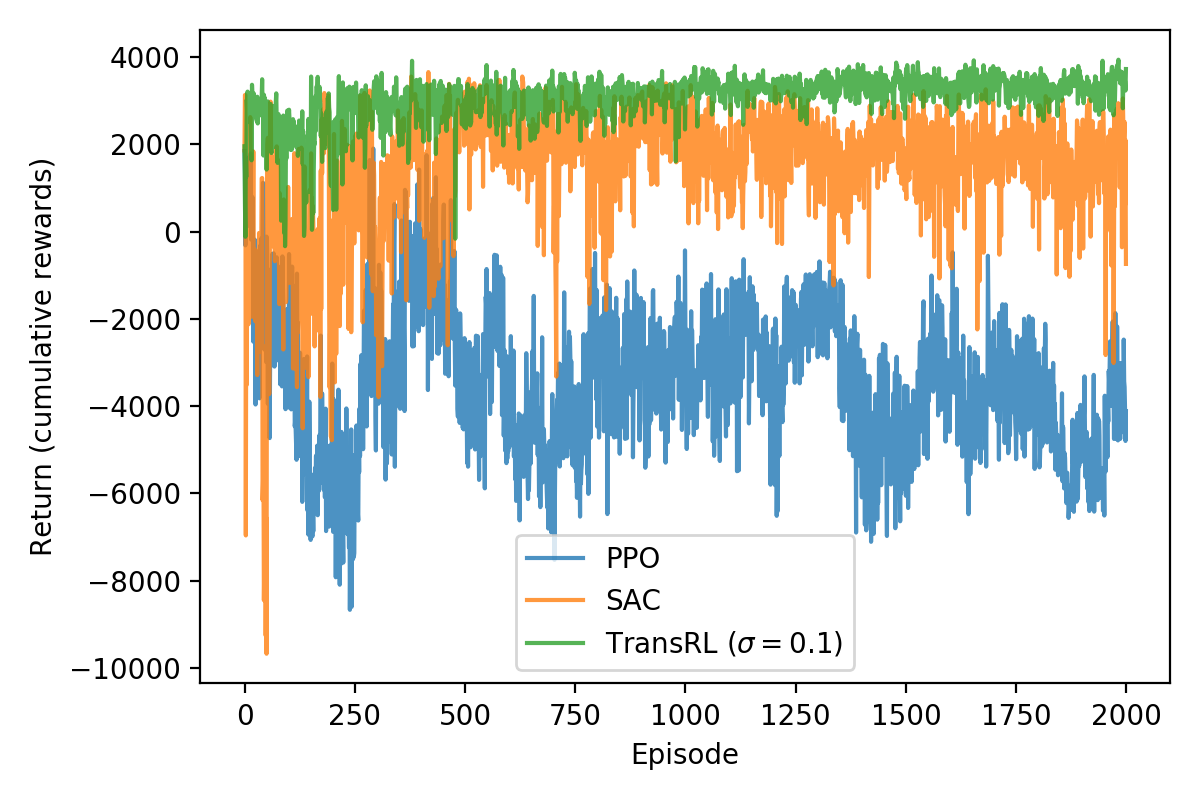}
         \caption{TransRL with an unreliability parameter of 0.10}
     \end{subfigure}
     \begin{subfigure}[b]{0.49\textwidth}
         \centering
         \includegraphics[width=\textwidth]{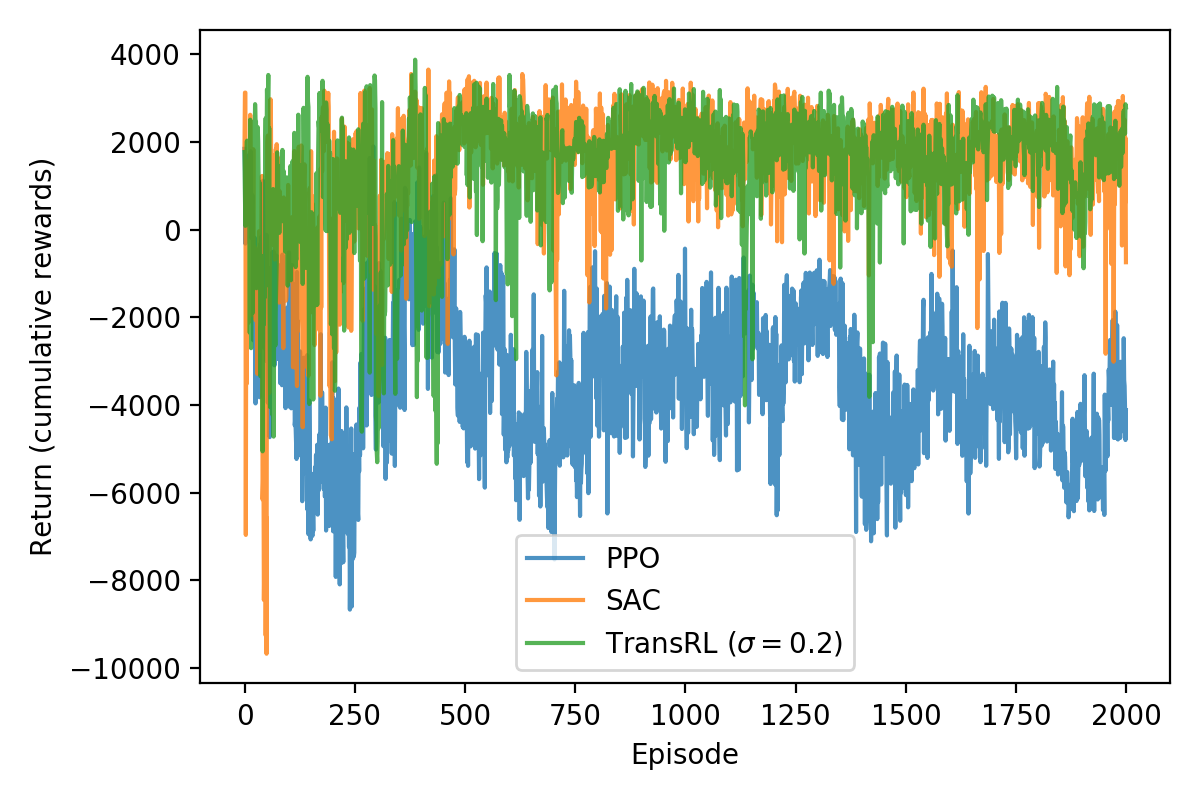}
         \caption{TransRL with an unreliability parameter of 0.20}
     \end{subfigure}
    \caption{Training on the corridor network}
    \label{fig: training on corridor network}
\end{figure}

\subsubsection{Control performances of online tests}
We tested trained RL and baselines for 100 test episodes with the same uncertain demands and model mismatch. The resultant total travel time is summarized in Table \ref{tab: TTT of corridor net} and also plotted in Figure \ref{fig: corridor network}. Though demands and the traffic model used in the model-based method (pre-DSO) are different from the actual demands and actual system dynamics, pre-DSO outperformed both model-free RL (SAC and PPO) in reducing total travel time. The average total travel time of PPO was even higher than the no-control baseline. 

With all three unreliability parameters, TransRL led to a lower average total travel time than PPO and SAC. When the unreliability parameter $\sigma$ of TransRL is 0.05 or 0.10, the average total travel time of TransRL was lower than pre-DSO. When $\sigma=0.20$, the average total travel time of pre-DSO was lower than TransRL. In summary, TransRL with $\sigma=0.10$ got the best control performance. The results also indicate careful tuning of the unreliability parameter can further enhance the performance of TransRL.

The reliability results of control methods on the corridor network are included in Table \ref{tab: CTTRaR of corridor net}. $\text{CTTRaR}_x$ of PPO and SAC are negative, which means, in the worst cases, PPO and SAC perform worse than the no-control baseline. Both pre-DSO and TransRL are reliable as $\text{CTTRaR}_x$ is high. Interestingly, while TransRL ($\sigma=0.10$) leads to the lowest average total travel time, TransRL ($\sigma=0.05$) is the most reliable. Results on this network indicate that, even when model-free RL is not very reliable, TransRL is still reliable because of the action constraints from the teacher policy.

\begin{table}[H]
\centering
\caption{Total travel time on the corridor network during 100 test episodes}
\label{tab: TTT of corridor net}
\begin{tabular}{@{}lll@{}}
\toprule
Total travel time       & Average         & SD      \\ \midrule
UE (no control)         & 117540.56       & 1256.21 \\
pre-DSO (model-based)                 & 113351.17       & 1347.67 \\
PPO (model-free RL)                  &  128814.06     & 3152.06  \\
SAC (model-free RL)                     & 116072.38       & 2720.69 \\
TransRL ($\sigma=0.05$) & \underline{112915.17}       & 1357.93 \\
TransRL ($\sigma=0.10$) & \textbf{112380.91} & 1550.75 \\
TransRL ($\sigma=0.20$) & 115536.98       & 1769.41 \\ \bottomrule
\end{tabular}
\end{table}

\begin{figure}[H]
    \centering
    \includegraphics[width=0.8\textwidth]{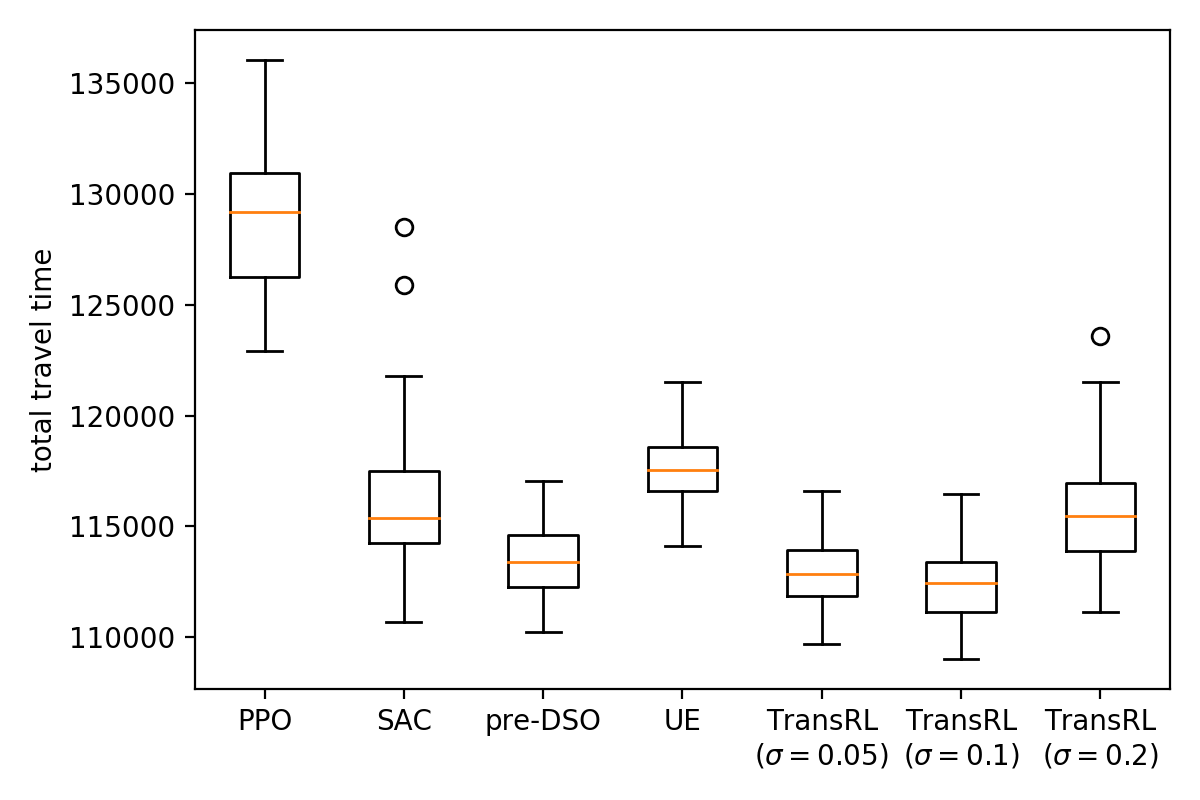}
    \caption{Total travel time on the corridor network during 100 test episodes}
    \label{fig: TTT on corridor network}
\end{figure}

\begin{table}[H]
\centering
\caption{$\text{CTTRaR}_{x}$ on the corridor network during 100 test episodes. The quantile $x$ is set to be 10\%, which means the top 10\% worst episodes.}
\label{tab: CTTRaR of corridor net}
\begin{tabular}{@{}ll@{}}
\toprule
   & $\text{CTTRaR}_{x}$ \\ \midrule 
pre-DSO (model-based) & \underline{4249.26
}   \\
PPO (model-free RL)     & -13089.30
 \\
SAC (model-free RL)    &  -2125.19
      \\
TransRL ($\sigma=0.05$) & \textbf{4418.89
}   \\ 
TransRL ($\sigma=0.10$) & 3882.36
   \\ 
TransRL ($\sigma=0.20$) &  1357.70
  \\ \bottomrule
\end{tabular}
\end{table}

\subsection{A large network with uncertain demands and traffic model mismatch}
The third network we experiment on is the Transportation Systems Management and Operations (TSMO) \# 1 network in Maryland, US. TSMO network contains a freeway I-70 and multiple US routes. As Figure \ref{fig: TSMO network} shows, there are 621 links, 361 nodes, and 182 OD pairs on the TSMO network. During morning peak hours, most travelers travel from west to east or south, and the eastbound of the I-70 is recurrently congested.

To estimate the demands of the TSMO network, we adopted the Dynamic OD demand Estimation (DODE) in \cite{ma2020estimating} using historical count data within the TSMO network. The DODE aims to solve dynamic demands to reproduce the link flows that match the historical count data. After estimating the dynamic demands, we generate demands using a Gaussian distribution $q_t^{ge} \sim \mathcal{N} (\mu_t^{ge}, \beta \mu_t^{ge})$, where $\mu_t^{ge}$ is the estimated demand of OD pair $ge$ at time $t$ and $\beta=0.10$ is the demand uncertainty parameter.

We consider two incident scenarios to evaluate the performances of the studied methods with unexpected incidents. Specifically, the accessible (approximated) traffic model $\Tilde{\mathcal{M}}$ is the network without incidents, while the actual system dynamics $\mathcal{M}$ is the network with unexpected incidents. Note, except for the incident links, the other attributes of $\Tilde{\mathcal{M}}$ and $\mathcal{M}$ are assumed to be the same to independently analyze the impacts of the incidents. The incident locations are marked using stars as Figure \ref{fig: TSMO network} shows. Both incidents happen on the freeway, and travelers cannot deviate from the incident location while observing the incident. Therefore, the incidents on the freeway will cause a nontrivial change in system dynamics. With the occurrence of incidents, we assume one lane is blocked and the free-flow speed decreases on the incident link.

In our optimal traffic routing problem, we consider a practical setting, where only a few links are observed and a few OD pairs are under control (e.g., the flow among those OD pairs may be influenced by dynamic message signs or naturally in a better geogrphic position to make route choices). More specifically, the number of observed links is 14, and the observed links are highlighted in Figure \ref{fig: TSMO network}. Among 182 OD pairs, there are 14 OD pairs of which flows can be influenced to re-route.

\begin{figure}[H]
    \centering
    \includegraphics[width=0.99\textwidth]{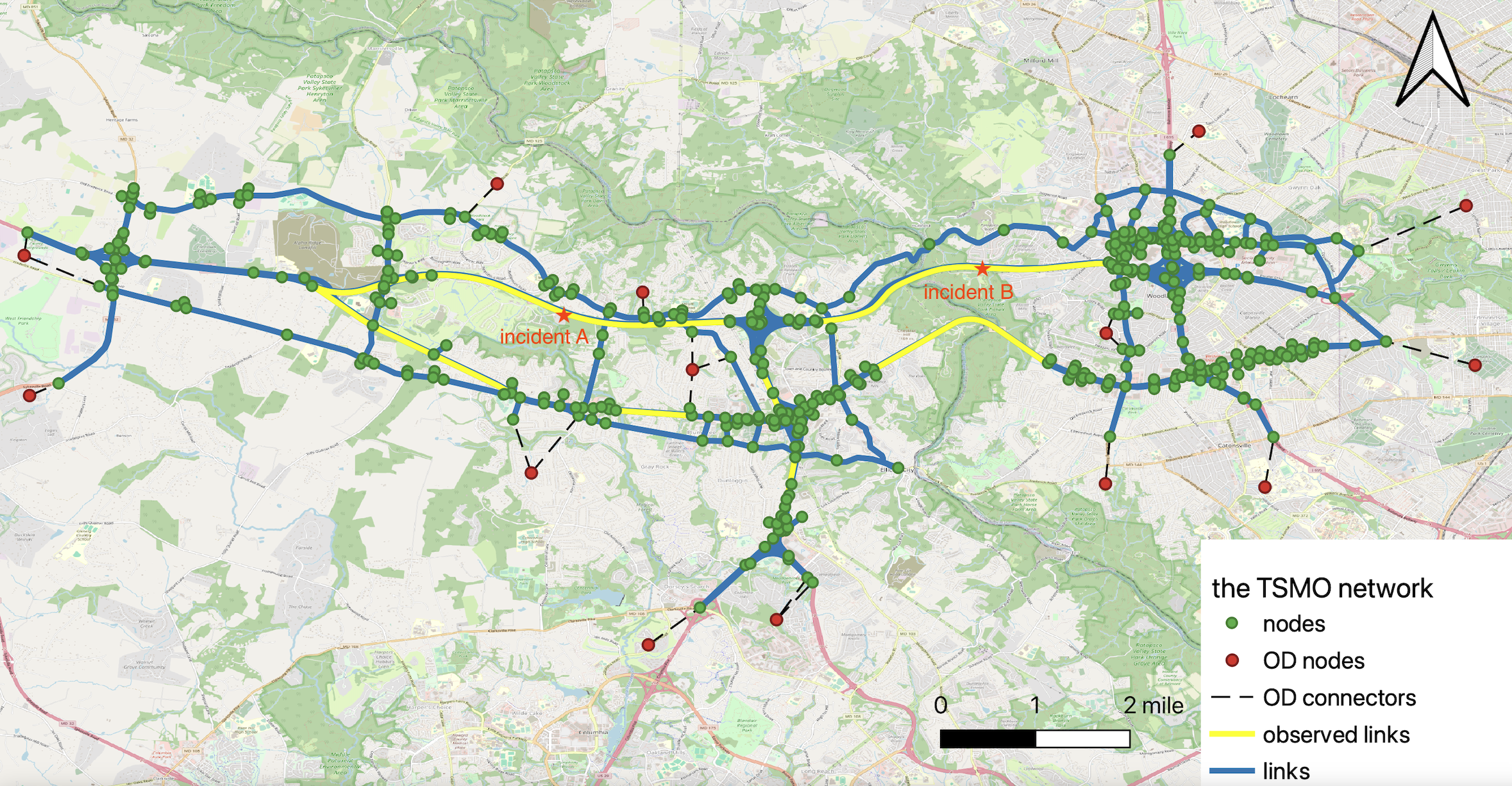}
    \caption{The TSMO network}
    \label{fig: TSMO network}
\end{figure}

\subsubsection{Calibration of dynamic demands}
Within the historical count data in the TSMO network, vehicles are classified into cars and trucks, we adopted the multi-class dynamic OD demand estimation (MCDODE) in \cite{ma2020estimating} to estimate car demands and truck demands. The results of MCDODE for the TSMO network are plotted in Figure \ref{fig: DODE}. The loss shows the MCDODE converged to a stable solution for both car data and truck data. The scatter plots in Figure \ref{fig: DODE} compare the observed count data and the estimated count data generated by the estimated demands. Most estimated data matched the observed data pretty well. The data points where the estimated data is significantly lower than the observed data can be attributed to the links located in the marginal area. After the estimation of the demands, we convert all vehicles into passenger car units (PCU) using passenger car equivalent (PCE) for subsequent experiments.

\begin{figure}[H]
    \centering
    \begin{subfigure}[b]{0.49\textwidth}
         \centering
         \includegraphics[width=\textwidth]{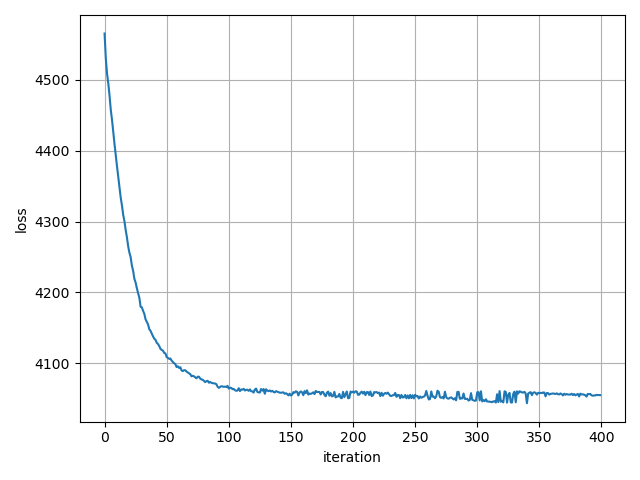}
         \caption{Loss for car count data}
     \end{subfigure}
          \begin{subfigure}[b]{0.49\textwidth}
         \centering
         \includegraphics[width=\textwidth]{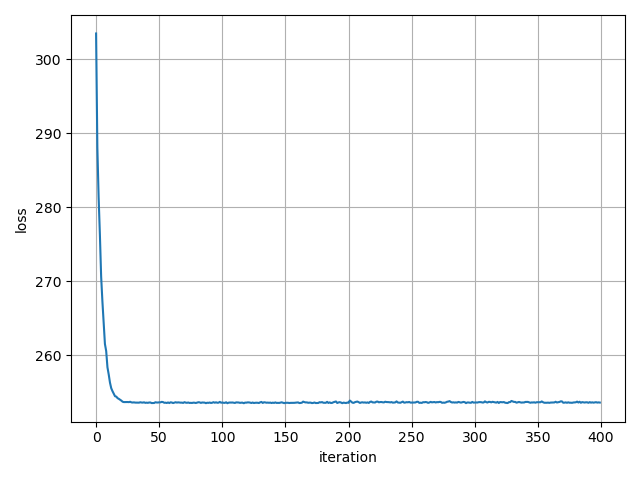}
         \caption{Loss for truck count data}
     \end{subfigure}
     \begin{subfigure}[b]{0.49\textwidth}
         \centering
         \includegraphics[width=\textwidth]{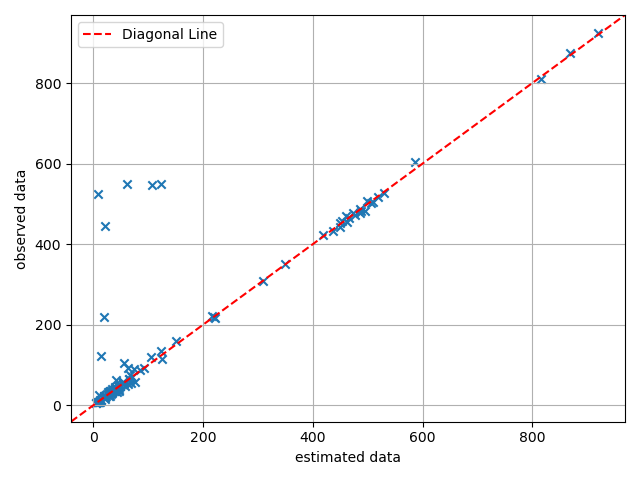}
         \caption{Car count data}
     \end{subfigure}
     \begin{subfigure}[b]{0.49\textwidth}
         \centering
         \includegraphics[width=\textwidth]{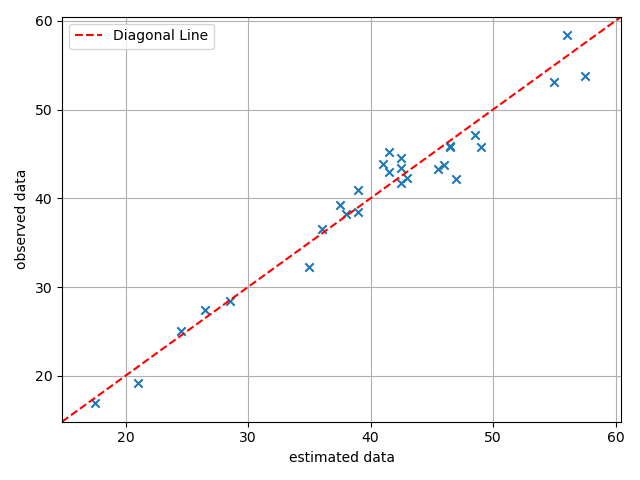}
         \caption{Truck count data}
     \end{subfigure}
    \caption{DODE for the TSMO network}
    \label{fig: DODE}
\end{figure}

\subsubsection{Training of reinforcement learning}
The training processes of PPO, SAC, and TransRL in both incident scenarios are compared in Figure \ref{fig: training on TSMO network}. In incident scenario A, compared with SAC, PPO converged faster and ended with higher returns. However, in incident scenario B, SAC converged faster than PPO, and PPO could not converge after 1000 episodes. This suggests none of the two model-free RLs dominates with the occurrence of unexpected incidents. On the other hand, TransRL converged faster and reached higher returns than PPO and SAC. Besides, TransRL in incident scenario A took more episodes to converge than TransRL in incident scenario B. This may be due to incident A impacting the system dynamics more significantly.

\begin{figure}[H]
    \centering
     \begin{subfigure}[b]{0.49\textwidth}
         \centering
         \includegraphics[width=\textwidth]{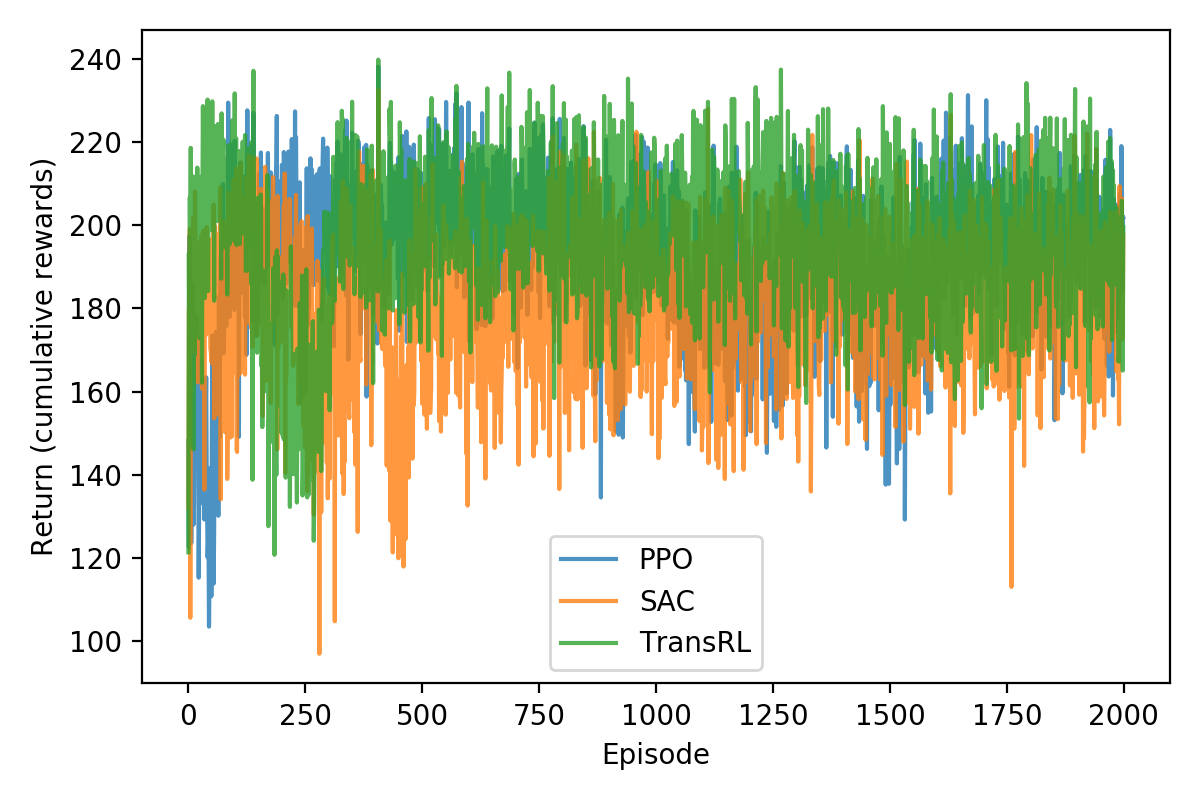}
         \caption{Incident scenario A}
     \end{subfigure}
     \begin{subfigure}[b]{0.49\textwidth}
         \centering
         \includegraphics[width=\textwidth]{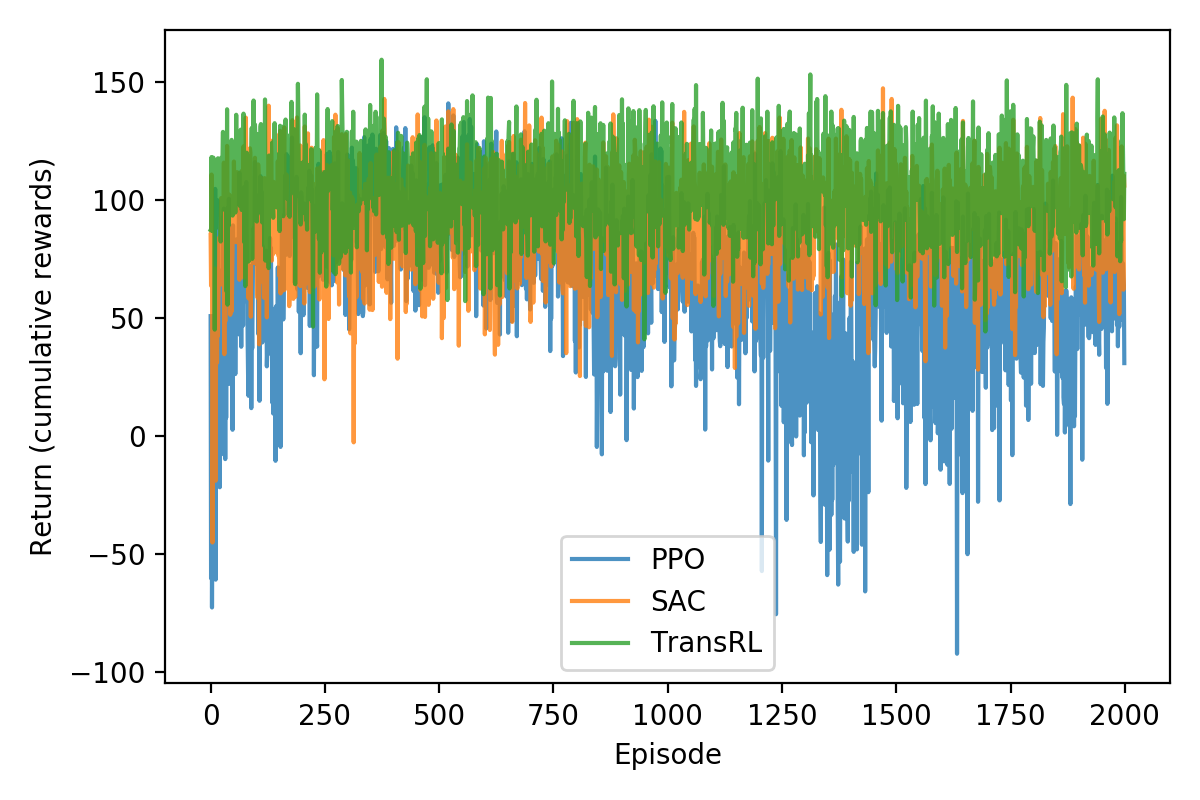}
         \caption{Incident scenario B}
     \end{subfigure}
    \caption{Training on the TSMO network}
    \label{fig: training on TSMO network}
\end{figure}

\subsubsection{Control performances of online tests}
The trained RL and pre-DSO were tested for 100 episodes with uncertain demands in the two incident scenarios. The total travel time is summarized in Table \ref{tab: TTT of TSMO net} and plotted in Figure \ref{fig: TTT on TSMO network}. The comparison between the model-based method (pre-DSO) and the two model-free RLs (i.e., PPO and SAC) is ambiguous. In incident scenario A, PPO and SAC led to lower total travel time than pre-DSO, and pre-DSO even increased total travel time compared with the no-control case. In incident scenario B,
pre-DSO outperformed PPO and SAC. This suggests that, under certain circumstances, incidents change the system dynamics so significantly that the model-based method may not improve the system. In contrast, TransRL outperformed pre-DSO, PPO, and SAC in both incident scenarios. This is interesting. The performance of TransRL in incident scenario A indicates that, even the model mismatch is non-trivial, TransRL is able to uncover a good policy after exploration and learning.

Table \ref{tab: CTTRaR of TSMO net} summarizes the reliability results. The reliability of model-free RL and pre-DSO is mixed. While PPO is more reliable than pre-DSO in incident scenario A, pre-DSO is more reliable than PPO in incident scenario B. In contrast, TransRL is the most reliable in both incident scenarios.

\begin{table}[H]
\centering
\caption{Total travel time on the TSMO network during 100 test episodes}
\label{tab: TTT of TSMO net}
\begin{tabular}{@{}lllll@{}}
\toprule
Scenario   & \multicolumn{2}{c}{Incident A} & \multicolumn{2}{c}{Incident B}  \\ \midrule
Total travel time & Average        & SD       & Average          & SD             \\ \midrule
UE (no control)    & 129028.95         & 1932.23 & 141198.15         & 3058.50 \\
pre-DSO (model-based)  & 135113.66          & 2066.20 & \underline{140529.34}          & 3054.75\\
PPO (model-free RL)     & \underline{126579.58}        & 2161.90 & 145846.63         & 3536.11\\
SAC (model-free RL)     & 128485.32          & 2717.12 & 141805.64          & 3569.62\\
TransRL (ours) & \textbf{126506.81} & 2002.43 & \textbf{140072.43} & 2627.40\\ \bottomrule
\end{tabular}
\end{table}

\begin{figure}[H]
    \centering
    \begin{subfigure}[b]{0.49\textwidth}
         \centering
         \includegraphics[width=\textwidth]{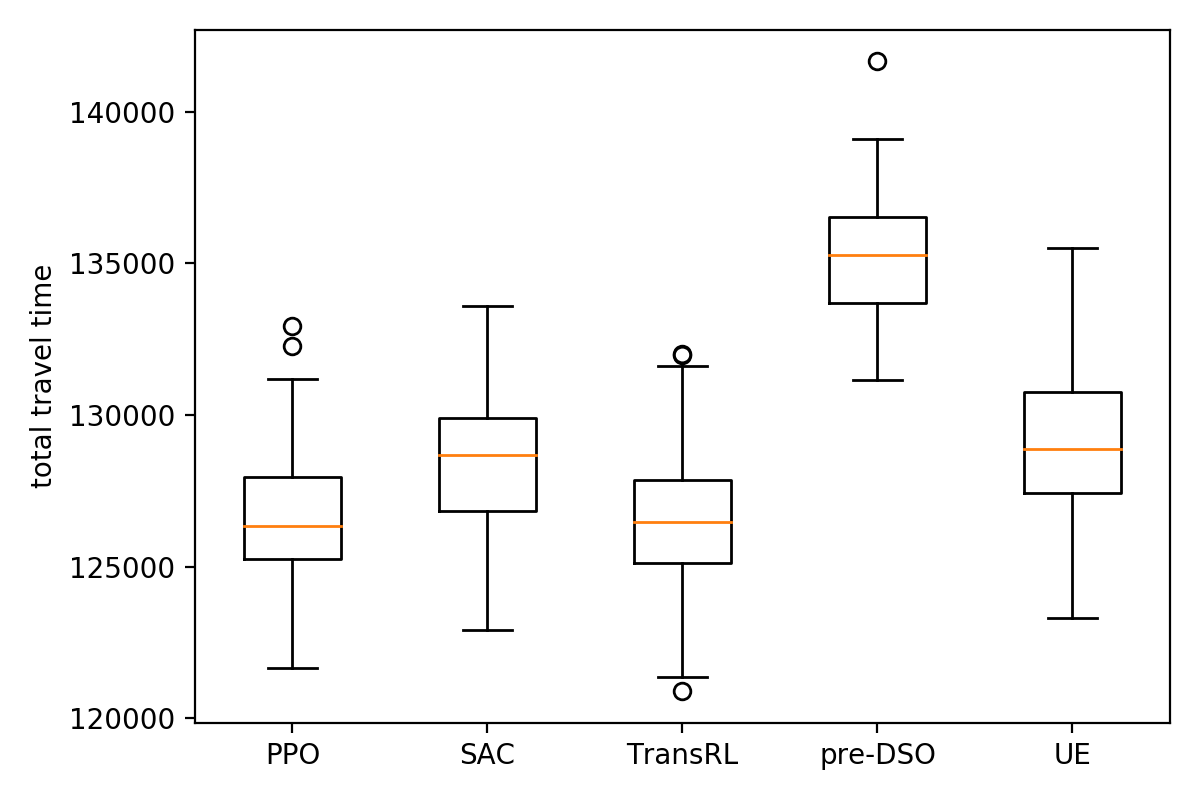}
         \caption{Incident scenario A}
     \end{subfigure}
         \begin{subfigure}[b]{0.49\textwidth}
         \centering
         \includegraphics[width=\textwidth]{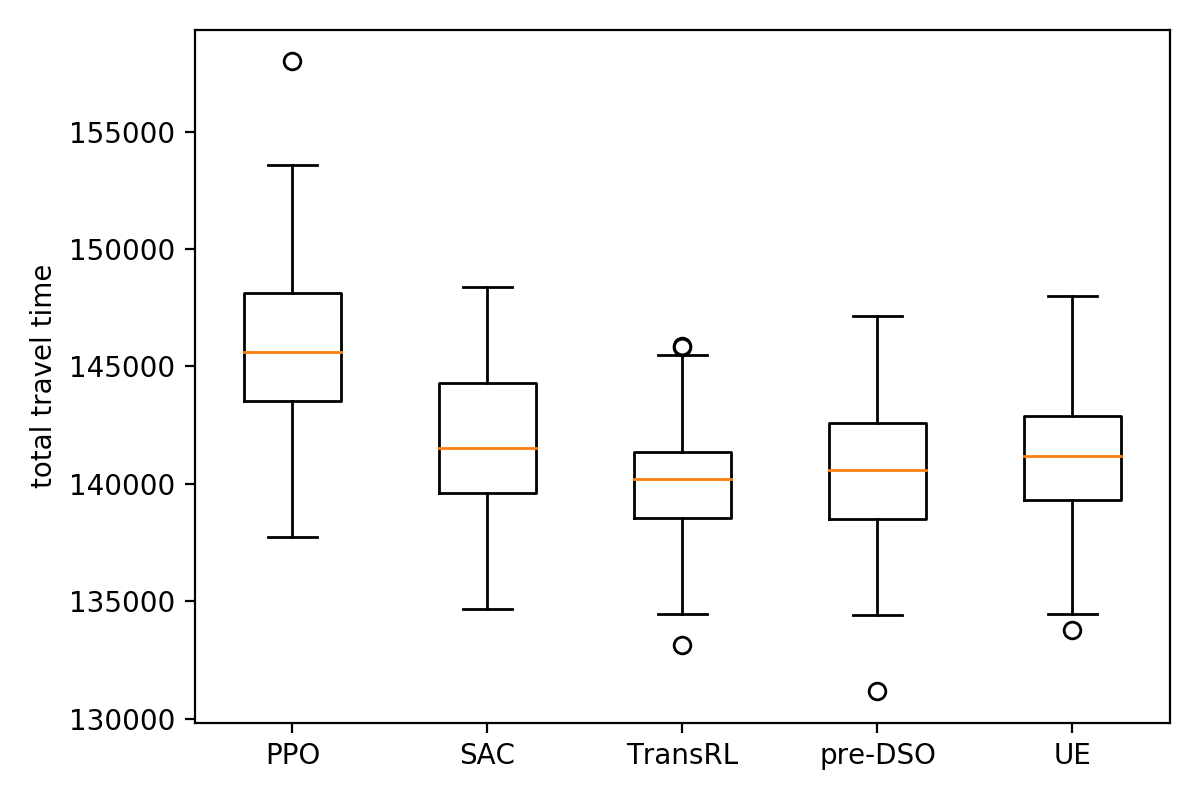}
         \caption{Incident scenario B}
     \end{subfigure}
    \caption{Total travel time on the TSMO network during 100 test episodes}
    \label{fig: TTT on TSMO network}
\end{figure}

\begin{table}[H]
\centering
\caption{$\text{CTTRaR}_{x}$ on the TSMO network during 100 test episodes. The quantile $x$ is set to be 10\%, which means the top 10\% worst episodes.}
\label{tab: CTTRaR of TSMO net}
\begin{tabular}{@{}lll@{}}
\toprule
   & Incident A & Incident B \\ \midrule
   & $\text{CTTRaR}_{x}$ & $\text{CTTRaR}_{x}$ \\ \midrule
pre-DSO (model-based) &  -6819.32 &  \underline{117.93}   \\
PPO (model-free RL)     & \underline{2871.87} & -3923.66 \\
SAC (model-free RL)    &  -1466.95 & -2614.68   \\
TransRL (ours) &  \textbf{3216.34} & \textbf{1506.50} \\ \bottomrule
\end{tabular}
\end{table}

\section{Conclusion}
\label{sec: conclusion}
This paper studies real-time system optimal traffic routing problems in general and sizable transportation networks with uncertainties. A small portion of vehicles are assumed to follow routing guidances by a centralized control algorithm. The objective of the control algorithm is to minimize the total expected travel time of all vehicles within the transportation network. We develop a novel RL model guided by a teacher policy that is derived directly from transportation domain models. Considering the realism of RL algorithms in real-world traffic operation, realistic testing scenarios are developed to test RL algorithms where the actual demands are stochastic and unknown, and there are model mismatches between the accessible traffic models that offline train RL models and the actual and unknown system dynamics testing RL models.

To incorporate the prior knowledge of the traffic models into RL, we proposed a novel RL framework TransRL. The reward of TransRL is composed of 1) a reward from the environment, and 2) a penalty reward that measures the distance between the current policy and the teacher policy deriving from the traffic model-based policy. An unreliability parameter is proposed to tune how much TransRL is concentrated on the traffic model-based policy. The experiments show that TransRL consistently outperforms the traffic model-based method and model-free RLs in reducing total travel time. Moreover, the actions of TransRL are more interpretable and reliable than model-free RLs.

We compare a traffic model-based RL with model-free RLs under different levels of demand uncertainties and model mismatches. The experiments suggest the traffic model-based RL, namely TransRL, performs well when the demand uncertainty level is low or medium and the model mismatch is mildlysignificant. Model-free RLs are agnostic to uncertain demands and system dynamics, so their performances are consistent along different demand uncertainties and model mismatches.


\section*{Acknowledgement}
This research is supported by US Department of Transportation Federal Highway Administration's Exploratory Advanced Research Award 693JJ321C000013. The contents of this paper reflect the views of the authors only, who are responsible for the facts and the accuracy of the information presented herein.

\bibliographystyle{plainnat}
\bibliography{ref}

\end{document}